\documentclass{article}

\PassOptionsToPackage{numbers, compress}{natbib}

\usepackage[main, final]{neurips_2026}

\usepackage[utf8]{inputenc} 
\usepackage[T1]{fontenc}    
\usepackage{hyperref}       
\usepackage{url}            
\usepackage{booktabs}       
\usepackage{multirow}       
\usepackage{amsmath}        
\usepackage{amsfonts}       
\usepackage{nicefrac}       
\usepackage{microtype}      
\usepackage{xcolor}         

\usepackage{graphicx}
\usepackage{wrapfig}
\usepackage{adjustbox}
\usepackage{amsthm}

\newtheorem{theorem}{Theorem}
\newtheorem*{theorem*}{Theorem}
\newtheorem{proposition}{Proposition}

\title{Cylindrical Geodesic Flow Matching for Quasiperiodic Physiological Signal Transformation}

\author{%
  Onur Selim Kilic \\
  Georgia Tech\\
  okilic3@gatech.edu
  \And
  Afra Nawar \\
  Georgia Tech\\
  anawar3@gatech.edu
  \And
  Cem Okan Yaldiz \\
  Georgia Tech\\
  cyaldiz3@gatech.edu
  \And
  Michael J.~Cho \\
  Georgia Tech\\
  mcho314@gatech.edu
  \And
  Ahmet Rasim Emirdagi \\
  Georgia Tech\\
  aemirdagi3@gatech.edu
  \And
  Demet Tangolar \\
  Georgia Tech\\
  dtangolar3@gatech.edu
  \AND
  Amirali Aghazadeh \\
  Georgia Tech\\
  amiralia@gatech.edu
  \And
  Amit J. Shah \\
  Emory University\\
  ajshah3@emory.edu
  \And
  Omer T. Inan \\
  Georgia Tech\\
  oinan3@gatech.edu
}

\begin{document}

\maketitle

\begin{abstract}
Paired translation between quasiperiodic physiological waveforms (i.e., recovering a target oscillatory signal from the source) is central to the interpretation of cardiovascular signals derived from wearables placed at different body locations.
This source-to-target mapping in these problems carries inherent geometric structure: the phase wraps around the cycle and must be treated as a circular variable, the amplitude remains strictly positive, and the beat-to-beat alignment can drift unpredictably across cycles and subjects. While deep neural networks have been used for phase estimation and complex-valued signal modeling, prior work does not explicitly learn phase transport between paired signals. Consequently, neither endpoint-supervised regression nor the standard affine path used in flow matching accounts for this phase--amplitude structure.
We introduce \emph{cylindrical geodesic flow matching} for paired cardiovascular waveform translation. We show that the standard affine path used in flow matching distorts intermediate amplitude and instantaneous frequency when interpolating between quasiperiodic signals; replacing it with a closed-form geodesic on the phase--amplitude cylinder eliminates these artifacts and converts each training pair into dense, geometry-consistent velocity supervision.
On zero-shot photoplethysmography and limited-support seismocardiography adaptation benchmarks, our method consistently outperforms interpolation baselines and matches or exceeds direct supervised prediction, reducing Hilbert Transform, $L_2$, and Dynamic Time Warping distance by up to ${\sim}15\%$ over the strongest competing baseline. These results suggest that bridge geometry is a critical inductive bias for flow matching on oscillatory signal translation.
\end{abstract}

\section{Introduction}
\label{sec:introduction}

Oscillatory physiological signals such as the photoplethysmogram (PPG) or seismocardiogram (SCG) are commonly used in wearable cardiovascular monitoring \cite{castaneda2018review,kim2022advances,inan2012bcg,castiglioni2007wearable,di2013wearable,inan2015review,etemadi2018wearable,inan2018novel,ganti2022stroke, tangolar2025enabling}. Their morphology, however, depends strongly on measurement location: PPG waveforms vary substantially across body sites \cite{hartmann2019quantitative,liu2021filtering,rahme2024comparative,nawar2025elucidating, tangolar2026establishing}, and SCG recordings vary across chest placements and subjects \cite{pandia2012resp,pandia2013freq,inan2013beat,ashouri2017misplacement,ashouri2018pep,zia2020quality,soliman2022tidal,nikbakht2023synthetic,lin2023realtime,munck2020multichannel,santucci2023waveform}. Many algorithms and reference standards are calibrated on canonical sensing locations, such as sternal SCG and finger PPG, where mechanical and vascular biomarkers are most reliably expressed. As wearables move toward more convenient but morphologically distinct sensing sites, accurate recovery of target-site morphology is necessary to preserve compatibility with validated diagnostic pipelines and support reliable downstream cardiovascular inference. Location-to-location waveform translation is therefore a natural paired problem: given a source signal at one site, recover the corresponding waveform (i.e., target) at another. 

This translation is difficult because physiological waveforms are quasiperiodic, that is, they exhibit repeating oscillatory structure whose phase, amplitude, and timing vary across cycles, subjects, and sensing locations. As a result, the source-to-target relationship is not a simple sample-by-sample transformation: phase is circular, amplitude is strictly positive, and inter-cycle alignment can drift unpredictably.

\begin{wrapfigure}[16]{r}{0.5\linewidth}
  \vspace{-0.75cm}
  \centering
  \includegraphics[width=\linewidth]{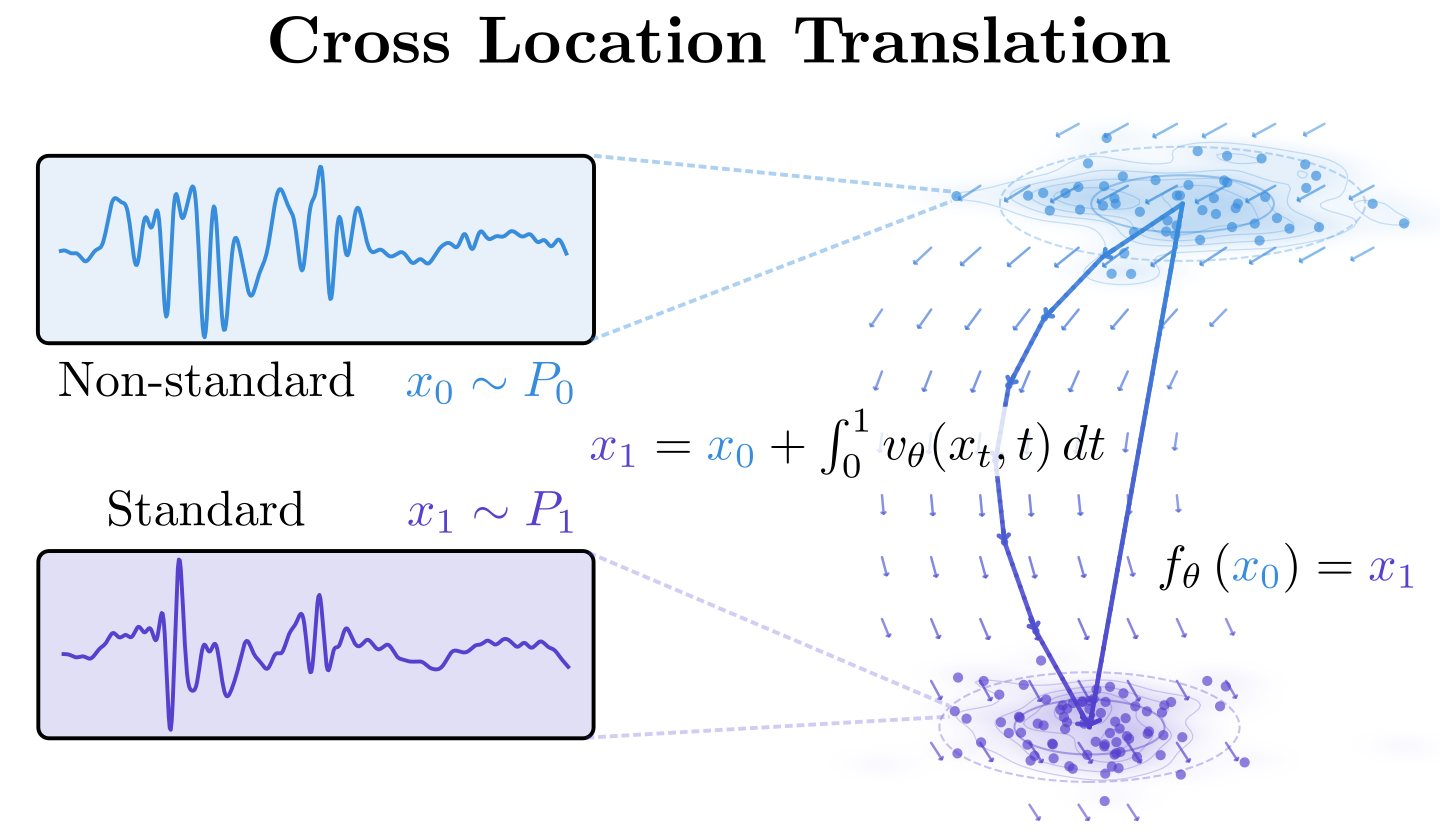}
  \caption{Cylindrical geodesic flow matching converts each source--target pair into dense velocity supervision along a geometry-aware bridge on the phase--amplitude cylinder, rather than relying on a single endpoint residual.}
  \label{fig:concept}
  \vspace{3cm}
\end{wrapfigure}

Direct supervised regression optimizes the prediction target but provides only sparse supervision. Each paired example contributes a single source-to-target residual, and the model must discover the underlying phase--amplitude geometry entirely from data. This is especially limiting when paired examples are scarce or cross-location alignment is ambiguous.

Conditional flow matching \cite{lipman2022flow,tong2023conditional} offers a complementary perspective. Each paired example generates supervision along an entire interpolating bridge, decomposing a difficult global transformation into a family of local velocity-regression problems \cite{song2020score}. Recent work has begun to apply advances in flow-matching to clinical and physiological time series, including irregular clinical trajectory modeling \cite{zhang2024trajectory}, electrocardiogram (ECG) synthesis \cite{bondar2025flowecg}, and PPG-conditioned vital-sign reconstruction \cite{suzuki2026penguin}.

Figure~\ref{fig:concept} illustrates this training view. Instead of learning from a single endpoint residual, the model receives dense supervision along the bridge between paired signals. Here, for example, ``non-standard'' may refer to a less-explored body location, such as the clavicle for SCG signals, whereas ``standard'' refers to a well-established location, such as the sternum. The resulting dense supervision yields a more favorable optimization landscape, but only if the bridge itself respects signal geometry. The standard affine path used in existing physiological flow-matching models cuts across the natural phase--amplitude structure of quasiperiodic signals, producing intermediate states with distorted amplitude and spurious instantaneous frequency, which can ultimately degrade the translated waveform.

To address this geometric mismatch, we introduce \emph{cylindrical geodesic flow matching}. Specifically, we represent each signal sample by its instantaneous phase $\phi$ and amplitude $a$ as $(\cos \phi, \sin \phi, a) \in S^1 \times \mathbb{R}_{>0}$, where $S^1$ denotes the unit circle. This embedding removes the discontinuity at phase wrapping while keeping amplitude explicit. On this cylinder, we then replace the standard affine interpolation path with a closed-form geodesic that respects phase periodicity by construction, eliminates the amplitude and frequency distortions inherent to the affine path, and yields an oracle vector field aligned with the phase--amplitude geometry of the task. We further parameterize the learned vector field intrinsically. Instead of predicting an unconstrained vector in $\mathbb{R}^3$, the model predicts phase and amplitude velocities, which are mapped to a tangent vector on the cylinder; combined with a per-step projection, this keeps the numerical trajectory on the manifold throughout integration.

The cylindrical representation, geodesic interpolation, and intrinsic tangent parameterization define a geometry-aware training framework for quasiperiodic waveform translation. These components address complementary aspects of the phase--amplitude structure, from endpoint representation and bridge construction to the learned dynamics along the path. On zero-shot PPG and limited-support SCG benchmarks, we demonstrate that the resulting framework consistently outperforms interpolation baselines and matches or exceeds direct supervised prediction across all reported metrics. Therefore, our main contributions are:
\begin{itemize}
    \item A cylindrical manifold representation $(\cos \phi, \sin \phi, a) \in S^1 \times \mathbb{R}_{>0}$ for physiological signals that eliminates phase-wrap discontinuities and preserves positive-amplitude structure. To our knowledge, no prior work explicitly learns phase transport between paired signals. Our formulation enables stable, direct learning of cross-signal phase–amplitude correspondences on the appropriate product manifold.
    \item A closed-form geodesic path and oracle velocity field on the cylinder for conditional flow matching, converting each training pair into dense, geometry-consistent supervision.
    \item Theoretical results showing that the standard affine path fails to preserve the phase--amplitude geometry of quasiperiodic cardiovascular signals, inducing amplitude shrinkage and instantaneous-frequency warping that can lead to lower quality translated signals, whereas the cylindrical geodesic preserves this geometry.
    \item Empirical validation across a range of settings, demonstrating that cylindrical geodesic flow matching outperforms competing methods.
\end{itemize}

\section{Related Work}
Supervised physiological signal translation and waveform regression form the closest task-level literature for grounding this work, as they learn direct mappings from paired source–target waveforms. The most relevant examples are paired biosignal reconstruction methods: Zhu \emph{et al.}\ map PPG to ECG in a transformed coefficient space, Tang \emph{et al.}\ develop subject-based PPG-to-ECG reconstruction, and Slapničar \emph{et al.}\ revisit the practical limits of ECG reconstruction from PPG under different alignment and generalization settings \cite{zhu2021heartactions,tang2022robust,slapnivcar2024fundamental}. Related paired-regression formulations also appear in acquisition-conditioned PPG correction, where Pham \emph{et al.}\ learn to restore distorted wrist PPG toward an ideal morphology \cite{pham2025reliable}, and in cross-sensor cardiac waveform recovery, where Skorić \emph{et al.}\ reconstruct ECG, impedance cardiography, blood pressure, and PPG waveforms from a single vibrational cardiography sensor \cite{skoric2025generative}. Although these settings differ in modality and sensor configuration, they share the same supervised endpoint-learning paradigm: paired examples train a predictor for the final waveform. These methods are strong when the paired mapping is informative, but they supervise only the endpoint and do not encode how a physiologically plausible source-to-target transformation should unfold. Flow matching is useful here not merely because it is continuous, but because each paired example induces supervision along an entire transport path. That path can be chosen to respect phase--amplitude geometry, so the model learns local, geometry-consistent dynamics instead of inferring the whole transformation from a single endpoint loss.

Flow matching recently emerged from the score-based and diffusion modeling literature \cite{ho2020ddpm,song2020score,song2020ddim} as a deterministic transport alternative that directly regresses the velocity field of a prescribed probability path \cite{lipman2022flow}. Conditional flow matching specializes this idea to paired endpoints, so that each source–target example supplies dense velocity supervision along an interpolating bridge \cite{tong2023conditional}. This conditional formulation has also begun to appear in clinical and physiological time-series applications. Trajectory Flow Matching adapts the framework to irregularly sampled clinical trajectories, learning continuous-time dynamics over incomplete patient records \cite{zhang2024trajectory}. FlowECG uses flow matching for efficient synthetic ECG generation, targeting data augmentation for downstream classifiers \cite{bondar2025flowecg}. PENGUIN combines flow matching with state-space modeling for PPG-conditioned reconstruction of ECG, respiration, and arterial blood pressure waveforms \cite{suzuki2026penguin}. These works establish flow matching as a viable tool in the physiological domain, yet each addresses trajectory modeling, waveform generation, or broad multi-signal reconstruction, and all rely on generic affine interpolation bridges. None asks whether the path itself should respect the geometry of the state space through which transport occurs, a question whose answer depends on the structure of the underlying signal morphology.

Riemannian flow matching demonstrates that the conditional framework generalizes naturally to manifold-valued data when geodesics replace straight-line interpolants, yielding geometry-consistent velocity fields for unconditional generation \cite{chen2024generalgeometries}. More broadly, manifold-aware generative models have shown improved sample quality when the transport map honors non-Euclidean data geometry \cite{mathieu2020riemannian}, and directional statistics provides the classical foundation for treating circular variables as intrinsically non-Euclidean objects whose arithmetic cannot be reduced to ordinary subtraction and averaging \cite{mardia2000directional,landler2018circular}. These results establish that path design should be informed by the manifold on which the data reside, yet existing applications target unconditional generation on abstract geometric spaces rather than conditional paired translation of physiological quasiperiodic waveforms. To instantiate a geometry-aware bridge for cardiovascular signal translation, one therefore needs a representation that exposes the relevant manifold structure of the waveforms themselves.

The analytic-signal formalism supplies exactly this representation. Gabor's original construction and Boashash's treatment of instantaneous phase and frequency \cite{gabor1946theory,135376} established the Hilbert transform as the standard tool for decomposing a real signal into an instantaneous amplitude envelope and an instantaneous phase, and extended Hilbert-transform variants have since been proposed to recover phase more robustly beyond ideal narrow-band regimes \cite{matsuki2023extended}. Functional-data analysis reinforces the importance of this decomposition by treating amplitude and phase as distinct sources of variation rather than absorbing both into a single pointwise discrepancy \cite{Marron_2015}. The same separation is operationalized in registration and alignment: multiresolution warping models phase variability explicitly across scales \cite{claeskens2010multiresolution}, and Fisher–Rao registration exploits the Riemannian structure of warping functions to align neural signals with pronounced phase variability \cite{zhao2020fisherrao}. Taken together, these results show that each sample of an analytic physiological signal naturally inhabits the product manifold $S^{1}\!\times\!\mathbb{R}_{>0}$, where the circular coordinate carries instantaneous phase and the positive-real coordinate carries instantaneous amplitude. 

Our study connects these four threads: by lifting paired cardiovascular waveforms into analytic-signal coordinates, we obtain the product-manifold state space on which Riemannian flow matching yields closed-form geodesic bridges and oracle velocities that decouple angular and amplitude dynamics, providing dense, geometry-consistent supervision that neither purely endpoint-supervised predictors nor generic affine-interpolation flow models supply, reducing HT Dist, $L_2$, and DTW by up to ${\sim}15\%$ over the strongest competing baseline on our SCG and PPG translation benchmarks.

\section{Background and Problem Setup}
\label{sec:background}
We consider paired signal transformation tasks with source waveform $x_0$ and target waveform $x_1$. For example, $x_0$ may be a waveform recorded at one sensing location and $x_1$ the corresponding waveform recorded at another, such as translating PPG from the sternum to the finger or SCG from the clavicle to the sternum. Let $\tau$ denote the original time variable and let $\mathcal{A}$ denote a generic analytic-signal transform applied to either waveform. For each signal $x_n$, where $n \in \{0,1\}$ indexes the source and target endpoints, the transform yields a complex analytic representation
\begin{equation}
    \zeta_n(\tau) = \mathcal{A}(x_n)(\tau) = a_n(\tau)e^{i\phi_n(\tau)}, \qquad n \in \{0,1\}.
    \label{eq:analytic-signal}
\end{equation}
Equation~\eqref{eq:analytic-signal} is the polar form of the complex coefficient $\zeta_n(\tau) \in \mathbb{C}$ at time $\tau$. Here $a_n(\tau) = |\zeta_n(\tau)| > 0$ is the instantaneous amplitude, $\phi_n(\tau) \in \mathbb{R}$ is an instantaneous phase angle measured in radians, and $e^{i\phi_n(\tau)}$ is the unit-modulus complex factor that carries phase only. Thus, each time sample is decomposed into a positive magnitude and an angular coordinate. The theory below is agnostic to the particular transform used to define $\zeta_n$. In experiments, we instantiate $\mathcal{A}$ with Hilbert Transform (HT) and Extended Hilbert Transform (EHT) \cite{matsuki2023extended}.

To compare paired signals at a common oscillatory resolution, we next decompose each analytic signal into phase and amplitude and evaluate both on the same uniform time grid. Let $K$ denote the number of retained samples, and let $s \in \{1,\dots,K\}$ denote the discrete sample index on that grid. After this step, $\phi_n(s)$ denotes the phase angle, in radians, at the $s$-th time-grid sample, and $a_n(s)$ denotes the corresponding amplitude. This produces the phase-amplitude state sequence
\begin{equation}
    h_n(s) = ([\phi_n(s)], a_n(s)) \in S^1 \times \mathbb{R}_{>0}.
    \label{eq:intrinsic-state}
\end{equation}
Equation~\eqref{eq:intrinsic-state} defines the intrinsic state at sample $s$ as a phase-amplitude pair, where $[\phi_n(s)]$ denotes phase modulo $2\pi$, with $[\phi_n(s)] \in S^1$ and $a_n(s) \in \mathbb{R}_{>0}$. Thus each coefficient lies in $S^1 \times \mathbb{R}_{>0}$. We then use the cylindrical embedding
\begin{equation}
    g_n(s) = \chi(h_n(s)) = \bigl(\cos \phi_n(s), \sin \phi_n(s), a_n(s)\bigr) \in \mathbb{R}^3.
    \label{eq:cylindrical-embedding}
\end{equation}
Equation~\eqref{eq:cylindrical-embedding} maps the wrapped phase to its unit-circle coordinates $(\cos \phi_n(s), \sin \phi_n(s))$ and appends amplitude as the third coordinate. This removes the artificial discontinuity between phase angles near $0$ and $2\pi$ while keeping amplitude explicit and positive. Geometrically, each coefficient becomes a point on an open cylinder in $\mathbb{R}^3$. Moreover, $\chi$ is a smooth embedding of $S^1 \times \mathbb{R}_{>0}$ into $\mathbb{R}^3$, so the cylindrical state space preserves the product structure of phase and amplitude in ambient coordinates.

The endpoint prediction problem is to estimate the target representation $g_1 = \{g_1(s)\}_{s=1}^K$ or the corresponding target waveform $x_1$ from $x_0$. A direct supervised model learns that mapping from paired endpoints alone. We instead map the same pair $(x_0, x_1)$ to $(g_0, g_1)$ and use that representation pair to generate time-indexed supervision along a structured bridge in phase-amplitude space.

\section{Method}

\label{sec:method}
\begin{figure}[t]
  \centering
  \includegraphics[width=\linewidth]{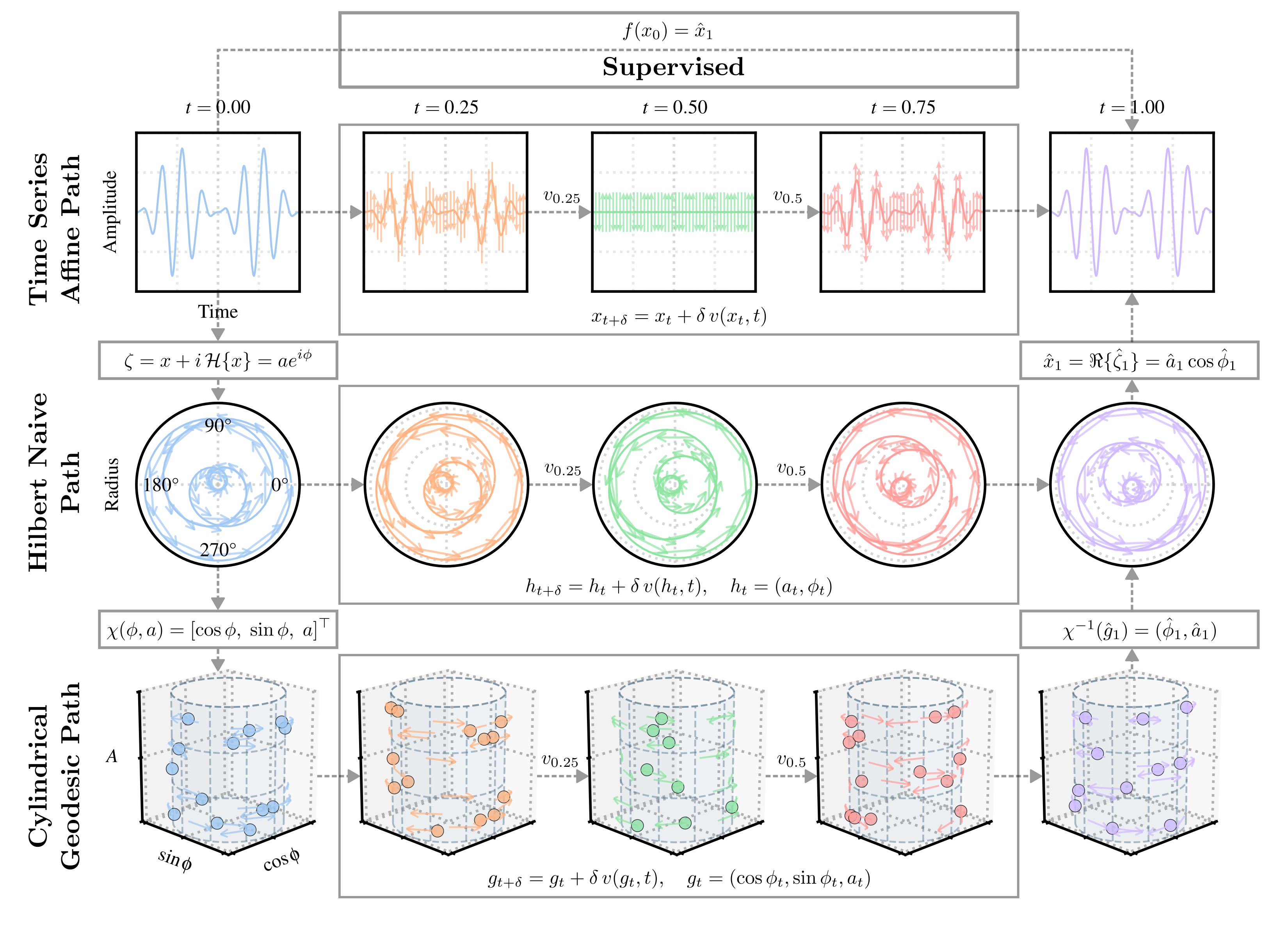}
  \caption{Method overview. A generic analytic-signal transform produces phase-amplitude coefficients, which are embedded on the cylinder, connected by a cylindrical geodesic, and used to define a tangent oracle field for conditional flow matching.}
  \label{fig:method}
\end{figure}

\subsection{Cylindrical Representation}
\label{subsec:cylindrical-representation}
Equation~\eqref{eq:cylindrical-embedding} embeds each phase-amplitude state into $\mathbb{R}^3$ as a point on an open cylinder. The angular coordinates $(\cos \phi, \sin \phi)$ accommodate the periodicity of wrapped phase on $S^1$ without introducing a discontinuity at $0 \equiv 2\pi$, while the third coordinate carries positive amplitude in $\mathbb{R}_{>0}$. This decomposition underlies both the geodesic interpolation rule and the intrinsic parameterization of the learned vector field.

Theorem~\ref{thm:embedding} shows that $\chi$ is a smooth isometric embedding of $S^1 \times \mathbb{R}_{>0}$ into $\mathbb{R}^3$ under the product metric $m = d\phi^2 + da^2$, and that on any bounded amplitude interval and any phase chart of width strictly less than $2\pi$, $\chi$ is also a smooth bi-Lipschitz change of variables. Modeling phase-amplitude coefficients and modeling their cylindrical embeddings are therefore equivalent up to a smooth change of coordinates on such local charts. In plain terms, the embedding provides a faithful 3D representation of the original variables, so we can model in the embedded space without changing the underlying local relationships in the data. The proof is deferred to Appendix~\ref{app:proofs}.

\subsection{Geodesic Interpolation on the Cylinder}
\label{subsec:geodesic-interpolation}
Given source and target states $h_0(s)=([\phi_0(s)], a_0(s))$ and $h_1(s)=([\phi_1(s)], a_1(s))$, we define the wrapped phase increment
\begin{equation}
    \Delta \phi(s) = \mathrm{wrap}\bigl(\phi_1(s)-\phi_0(s)\bigr) \in (-\pi,\pi],
    \label{eq:wrapped-phase}
\end{equation}
where $\mathrm{wrap}(\theta)$ denotes the unique representative of $\theta \bmod 2\pi$ in $(-\pi,\pi]$. We also define the amplitude increment
\begin{equation}
    \Delta a(s) = a_1(s)-a_0(s).
    \label{eq:delta-amplitude}
\end{equation}
Under the product metric $m = d\phi^2 + da^2$ on $S^1 \times \mathbb{R}_{>0}$, we use the shortest-arc, constant-speed interpolation
\begin{equation}
    \gamma_C(t;s) = \bigl([\phi_0(s) + t\,\Delta \phi(s)], \; a_0(s) + t\,\Delta a(s)\bigr), \qquad t \in [0,1].
    \label{eq:geodesic-intrinsic}
\end{equation}
Its embedded form is
\begin{equation}
    g_t(s) = \bigl(\cos \phi_t(s), \sin \phi_t(s), a_t(s)\bigr),
    \label{eq:geodesic-embedded}
\end{equation}
with $\phi_t(s) = \phi_0(s) + t\,\Delta \phi(s)$ and $a_t(s) = a_0(s) + t\,\Delta a(s)$. Proposition~\ref{prop:geodesic} shows that, for each fixed sample $s$, this path is a constant-speed minimizing geodesic joining $h_0(s)$ to $h_1(s)$, with intrinsic length
\[
    L_C(s) = \sqrt{(\Delta \phi(s))^2 + (\Delta a(s))^2},
\]
and that it is unique whenever $|\Delta \phi(s)| < \pi$. The proof is deferred to Appendix~\ref{app:proofs}.

\subsection{Oracle Geodesic Vector Field}
\label{subsec:oracle-field}
Differentiating Eq.~\eqref{eq:geodesic-embedded} yields the oracle velocity field
\begin{equation}
    v_t^\star(s)
    =
    \frac{d}{dt}g_t(s)
    =
    \begin{bmatrix}
        -\sin \phi_t(s)\,\Delta \phi(s) \\
        \phantom{-}\cos \phi_t(s)\,\Delta \phi(s) \\
        \Delta a(s)
    \end{bmatrix}.
    \label{eq:oracle-field}
\end{equation}
Proposition~\ref{prop:geodesic} shows that this oracle field has constant Euclidean norm along the bridge:
\[
    \|v_t^\star(s)\|_2^2 = (\Delta \phi(s))^2 + (\Delta a(s))^2.
\]
Because $v_t^\star$ is defined for every $(s,t)$, each paired example generates dense supervision along the entire bridge rather than a single endpoint residual. We exploit this by training a conditional vector field $v_\theta$ to minimize the flow matching loss over sampled bridge points:
\begin{equation}
    \mathcal{L}_{\mathrm{FM}}
    =
    \mathbb{E}_{(x_0,x_1),\, s,\, t}
    \Bigl[
        \bigl\|
        v_\theta(g_t(s), t, x_0) - v_t^\star(s)
        \bigr\|_2^2
    \Bigr].
    \label{eq:fm-loss}
\end{equation}
The constant-norm property ensures that the regression target has uniform scale across the bridge, so no region of $t$ dominates the loss.

\subsection{Comparison with Affine Interpolation}
\label{subsec:why-geodesic}
Affine interpolation is the default Euclidean bridge in conditional flow matching, making it the natural baseline for our geodesic construction. To isolate the effect of bridge geometry alone, we fix the endpoints, consider a single sample, and specialize to the equal-amplitude case $a_0 = a_1 = a > 0$. Let $w_0 = a e^{i\phi_0}$ and $w_1 = a e^{i(\phi_0 + \Delta \phi)}$, and define the cylindrical and affine paths in complex form as
\[
    \gamma_C(t) = a e^{i(\phi_0 + t\,\Delta \phi)},
    \qquad
    \gamma_{\mathrm{aff}}(t) = (1-t)w_0 + t\,w_1.
\]
To quantify how faithfully each path preserves the phase--amplitude structure of the endpoints, we measure two distortion functionals for any nonvanishing path $\gamma(t) = a_t e^{i\phi_t}$:
\[
    D_{\mathrm{amp}}(\gamma) = \int_0^1 (a_t-a)^2\,dt,
    \qquad
    D_{\mathrm{freq}}(\gamma) = \int_0^1 \bigl(\dot{\phi}_t-\Delta \phi\bigr)^2\,dt.
\]
\begin{theorem*}[Bridge fidelity and path-length optimality, informal]
\label{thm:main-informal}
With the notation above, suppose $\Delta\phi \neq 0$.
\begin{enumerate}
\item (Fidelity; $a_0 = a_1 = a$)
The cylindrical geodesic incurs zero distortion: $D_{\mathrm{amp}}(\gamma_C) = 0$ and $D_{\mathrm{freq}}(\gamma_C) = 0$.
The affine path incurs strictly positive distortion for every $0 < |\Delta\phi| < \pi$, with leading-order behavior as $\Delta\phi \to 0$:
\[
    D_{\mathrm{amp}}(\gamma_{\mathrm{aff}}) = \frac{a^2\,\Delta\phi^4}{120} + O(\Delta\phi^6), \qquad D_{\mathrm{freq}}(\gamma_{\mathrm{aff}}) = \frac{\Delta\phi^6}{180} + O(\Delta\phi^8).
\]
At the boundary $|\Delta\phi| = \pi$, the affine midpoint amplitude vanishes and its phase becomes undefined.

\item (Path length)
The cylindrical geodesic is length-minimizing with
\[
    L_C = \sqrt{\Delta\phi^2 + \Delta a^2}.
\]
The intrinsic lift of the affine chord, whenever well-defined, satisfies $L_{\mathrm{aff}} \ge L_C$, with equality if and only if $\Delta\phi = 0$.
\end{enumerate}
Formal statements and proofs appear in Appendix~\ref{app:proofs} (Theorems~\ref{thm:fidelity} and~\ref{thm:length}).
\end{theorem*}

In cardiovascular signals, phase controls beat timing and amplitude controls morphology: the affine bridge therefore creates intermediate states with attenuated pulses and warped cycle progression, yielding supervision that is less physiologically plausible. The cylindrical geodesic avoids this detour, ensuring that the supervision in Eq.~\eqref{eq:fm-loss} reflects the intrinsic phase--amplitude interpolation rather than geometric artifacts of an ambient chord.

\subsection{Intrinsic Tangent Parameterization}
\label{subsec:intrinsic-parameterization}
If the model predicts an unrestricted vector in $\mathbb{R}^3$, numerical integration can push the state off the cylinder, accumulating manifold drift over successive steps. We address this with two complementary mechanisms. First, we parameterize the learned dynamics intrinsically. For a state $g=(\cos \phi, \sin \phi, a)$, the model outputs an angular velocity $\omega_\theta$ and an amplitude velocity $\dot{a}_\theta$, which are mapped to the ambient space via
\begin{equation}
    \hat{v}_\theta(g,t,x_0)
    =
    \begin{bmatrix}
        -\sin \phi \, \omega_\theta(g,t,x_0) \\
        \phantom{-}\cos \phi \, \omega_\theta(g,t,x_0) \\
        \dot{a}_\theta(g,t,x_0)
    \end{bmatrix}.
    \label{eq:intrinsic-pushforward}
\end{equation}
The first two coordinates are automatically tangent to $S^1$, so the predicted velocity has no radial component and the continuous-time flow of Eq.~\eqref{eq:intrinsic-pushforward} remains on the cylinder by construction. Under an explicit discrete solver, however, straight-line steps along a tangent direction leave the curved surface, producing off-manifold drift that is second-order in the step size. To eliminate this residual drift, we apply a second mechanism: after each solver substep we project the state back onto the cylinder by renormalizing the circular component to unit norm. Together, the tangent parameterization and the per-substep projection keep the numerical trajectory on the manifold throughout integration. This parameterization also mirrors the structure of the oracle target in Eq.~\eqref{eq:oracle-field}, where the angular and amplitude components are likewise decoupled, ensuring that the model capacity is devoted to learning the intrinsic phase and amplitude velocities rather than correcting off-manifold drift.

\section{Experiments}
\label{sec:experiments}

\subsection{Setup}
\label{subsec:experimental-setup}

\paragraph{Tasks and motivation.}
We evaluate cylindrical geodesic flow matching on two paired cardiovascular signal-translation tasks: sternum-to-finger PPG and clavicle-to-sternum SCG, each conditioned on the simultaneously recorded ECG. The targets are clinically canonical sites: finger PPG is the reference modality for pulse oximetry and vascular biomarker extraction, and sternal SCG is the validated location for cardiac mechanical timings such as the aortic-opening and aortic-closing intervals and the pre-ejection period. The sources, by contrast, are convenient wearable placements whose morphology departs substantially from these canonical sites. Recovering the target-site waveform from the source-site recording therefore preserves compatibility with established diagnostic pipelines while enabling unobtrusive continuous monitoring. This is a practically important regime for which paired translation is a natural formulation.

\paragraph{Data and protocol.}
Both datasets were collected at the authors' institution under separate IRB-approved protocols at $500$\,Hz and include baseline and stressor/recovery phases. The PPG cohort consists of 20 single-session subjects, segmented into sliding windows; the SCG cohort consists of 13 subjects across 24 sessions, segmented at the beat level. We use subject-disjoint $50/25/25$ train/validation/test splits repeated over four random seeds. PPG is evaluated zero-shot on held-out subjects with no adaptation. SCG, in contrast, exhibits substantial session-to-session variability driven by sensor placement, coupling, and body-habitus differences, so zero-shot transfer is not realistic and a small per-session calibration is the operationally relevant regime. For SCG, we adapt to each held-out session using a strictly causal $512$-beat prefix drawn from the earliest portion of the session (resting baseline and, where needed, the initial recovery phase) and evaluate on the chronological remainder, which spans subsequent stressor recoveries; this is stricter than baseline-only evaluation because adaptation is fixed early while the model must track later condition changes.

\paragraph{Metrics and baselines.}
We report Hilbert Transform Distance (HT Dist; joint amplitude–phase error in the analytic-signal domain), $L_2$ error (root-mean-square waveform error), and Dynamic Time Warping error (DTW; alignment-tolerant waveform mismatch); lower is better. All three are computed under both the standard HT and EHT. We compare four learned models sharing the same 1D U-Net backbone, optimizer, and training schedule: (i) a \emph{Hilbert naive path}, which interpolates linearly between source and target in the unwrapped phase-amplitude coordinates; (ii) a \emph{time-series affine path}, which interpolates linearly in the raw waveform domain; (iii) a \emph{direct supervised predictor}, which regresses the target endpoint from the source; and (iv) our \emph{cylindrical geodesic path}. The first two are treated as interpolation baselines that isolate the effect of bridge geometry within conditional flow matching; the third serves as a strong endpoint-prediction baseline. The source-to-target discrepancy without any learned translation is reported as a task-difficulty reference. Formal metric definitions, exact splits, hyperparameters, and compute environment are in Appendix~\ref{app:implementation}.

\subsection{Results}
\label{subsec:main-results}

\begin{table*}[t]
    \centering
    \small
    \setlength{\tabcolsep}{6pt}
    \begin{tabular}{@{}llccc@{}}
        \toprule
        Transform & Method / Reference & HT Dist $\downarrow$ & $L_2$ $\downarrow$ & DTW $\downarrow$ \\
        \midrule
        \multirow{5}{*}{HT} & Source Reference & $1.5389 \pm 0.0998$ & $1.2086 \pm 0.0512$ & $10.2730 \pm 0.5396$ \\
         & Hilbert Naive Path & $1.3752 \pm 0.1143$ & $0.9835 \pm 0.0668$ & $7.9733 \pm 0.5303$ \\
         & Direct Supervised Path & $0.9223 \pm 0.3275$ & $0.6867 \pm 0.2261$ & $5.0619 \pm 1.3302$ \\
         & Time Series Affine Path & $0.9073 \pm 0.2933$ & $0.6951 \pm 0.1869$ & $5.6048 \pm 1.2262$ \\
         & Cylindrical Geodesic Path & $\boldsymbol{0.8294 \pm 0.3619}$ & $\boldsymbol{0.5944 \pm 0.2310}$ & $\boldsymbol{4.7902 \pm 1.3007}$ \\
        \midrule
        \multirow{5}{*}{EHT} & Source Reference & $1.5634 \pm 0.0981$ & $1.1975 \pm 0.0545$ & $10.0716 \pm 0.5789$ \\
         & Hilbert Naive Path & $1.4726 \pm 0.0857$ & $1.0220 \pm 0.0452$ & $8.4820 \pm 0.7937$ \\
         & Direct Supervised Path & $\boldsymbol{0.9777 \pm 0.2739}$ & $0.7391 \pm 0.1345$ & $7.0076 \pm 1.3065$ \\
         & Time Series Affine Path & $1.0885 \pm 0.2708$ & $0.7591 \pm 0.1707$ & $6.1652 \pm 0.9439$ \\
         & Cylindrical Geodesic Path & $1.0110 \pm 0.2542$ & $\boldsymbol{0.7237 \pm 0.1514}$ & $\boldsymbol{5.9589 \pm 0.9057}$ \\
        \bottomrule
    \end{tabular}
    \caption{PPG zero-shot results (mean $\pm$ std over 4 runs, each with a different random seed and held-out test-subject split). Under HT, the cylindrical geodesic path performs best on all metrics. Under EHT, it performs best on $L_2$ and DTW, while direct supervision is best on HT Dist.}
    \label{tab:ppg-main-results}
\end{table*}

\begin{table*}[h]
    \centering
    \small
    \setlength{\tabcolsep}{6pt}
    \begin{tabular}{@{}llccc@{}}
        \toprule
        Transform & Method / Reference & HT Dist $\downarrow$ & $L_2$ $\downarrow$ & DTW $\downarrow$ \\
        \midrule
        \multirow{5}{*}{HT} & Source Reference & $1.4098 \pm 0.0628$ & $1.1932 \pm 0.0256$ & $5.8618 \pm 0.1735$ \\
         & Hilbert Naive Path & $1.7230 \pm 0.0478$ & $1.3031 \pm 0.0810$ & $6.3553 \pm 0.3431$ \\
         & Direct Supervised Path & $1.0419 \pm 0.0734$ & $0.7103 \pm 0.0293$ & $4.9158 \pm 0.1372$ \\
         & Time Series Affine Path & $0.9352 \pm 0.0181$ & $0.6344 \pm 0.0260$ & $4.4071 \pm 0.2501$ \\
         & Cylindrical Geodesic Path & $\boldsymbol{0.8936 \pm 0.0219}$ & $\boldsymbol{0.6181 \pm 0.0231}$ & $\boldsymbol{4.2618 \pm 0.2882}$ \\
        \midrule
        \multirow{5}{*}{EHT} & Source Reference & $1.4665 \pm 0.0546$ & $1.2111 \pm 0.0235$ & $6.0423 \pm 0.1496$ \\
         & Hilbert Naive Path & $1.7462 \pm 0.0146$ & $1.2976 \pm 0.0456$ & $6.2527 \pm 0.2739$ \\
         & Direct Supervised Path & $1.2437 \pm 0.0282$ & $0.9129 \pm 0.0377$ & $5.1282 \pm 0.2807$ \\
         & Time Series Affine Path & $1.1771 \pm 0.0335$ & $0.8551 \pm 0.0278$ & $5.0171 \pm 0.2362$ \\
         & Cylindrical Geodesic Path & $\boldsymbol{1.0565 \pm 0.0270}$ & $\boldsymbol{0.8059 \pm 0.0282}$ & $\boldsymbol{4.7527 \pm 0.2689}$ \\
        \bottomrule
    \end{tabular}
    \caption{SCG limited-support adaptation results (mean $\pm$ std over 4 runs, each with a different random seed and held-out test-subject split). The cylindrical geodesic path performs best on all three metrics under both HT and EHT.}
    \label{tab:scg-main-results}
\end{table*}

Tables~\ref{tab:ppg-main-results} and~\ref{tab:scg-main-results} summarize results for both datasets under both transform settings. On PPG zero-shot transfer (Table~\ref{tab:ppg-main-results}), the cylindrical geodesic path is the best on five of the six metric--transform cells, lowering HT Dist, $L_2$, and DTW by ${\sim}9$--$15\%$ over the time-series affine bridge under HT, and remaining the best on $L_2$ and DTW under EHT, where direct supervision is marginally better on HT Dist. On SCG limited-support adaptation (Table~\ref{tab:scg-main-results}), the advantage is uniform: the cylindrical geodesic path is the best on all six cells, with ${\sim}13\%$ reductions across all three metrics over direct supervision under HT and ${\sim}10$--$15\%$ reductions over both baselines under EHT. The Hilbert naive path, which trains on amplitude and phase without our cylindrical embedding, lags substantially on PPG and is worse than the source reference on SCG, indicating that direct phase--amplitude training on physiological signals is not viable without the geometric structure introduced by the cylindrical mapping. Inference uses a Heun solver~\cite{karras2022elucidating} with $100$ steps and a unit-norm projection of the circular component after each substep (Section~\ref{subsec:intrinsic-parameterization}). Support-size and solver-step ablations (Appendices~\ref{app:support-ablation} and~\ref{app:solver-ablation}) further show that these gains persist at $32$--$64$ support beats and ${\le}16$ solver steps, indicating that the geometric advantage of the cylindrical bridge is robust to substantially smaller adaptation budgets and coarser integration than used in the main tables.

\section{Discussion}
\label{sec:discussion}

To our knowledge, this is the first work that learns paired waveform translation directly in instantaneous phase--amplitude coordinates rather than as a sample-wise time-series regression. The Hilbert naive baseline shows the cost of ignoring this geometry: trained on the same $(\phi, a)$ representation but treating the circular coordinate as Euclidean, it lags on PPG and is worse than the identity-map source reference on SCG. Thus, phase--amplitude training is not, by itself, useful. Without explicit encoding of periodicity and positivity, the learned dynamics inherit the same discontinuities that the analytic representation was meant to remove. The cylindrical embedding, geodesic bridge, and intrinsic tangent parameterization are what make this representation work, and the gap between the Hilbert naive and cylindrical paths measures that.

Two predictions from Section~\ref{subsec:why-geodesic} are supported by the experiments. \emph{Bridge fidelity:} under the standard Hilbert transform, where phase and amplitude are most directly recoverable, the cylindrical path achieves the best performance on every metric--task pair, with margins over the affine bridge in line with the $O(\Delta\phi^4)$ amplitude and $O(\Delta\phi^6)$ frequency artifacts predicted by Theorem~\ref{thm:fidelity}. \emph{Supervisory density:} the advantage over direct endpoint supervision is largest on SCG limited-support adaptation and holds under reduced support sizes and coarser solvers (Appendices~\ref{app:support-ablation},~\ref{app:solver-ablation}). This matches the picture in which the geometric bridge supplies a full trajectory of geometry-consistent velocity targets per pair, which matters most when paired endpoints are scarce.

Our paper has three main limitations. First, the geodesic uses the flat product metric $d\phi^2 + da^2$; a log-amplitude or Fisher--Rao-style envelope metric may be a better fit for modalities with strong multiplicative envelope variability, and we have not mapped out which choice is preferable in which regime. Second, the analytic-signal representation has known weak points near low-amplitude segments, where the phase becomes ill-defined; the EHT mitigates but does not eliminate this, and our framework inherits whatever residual error remains at those points. Third, our evaluation covers two cardiovascular translation tasks at a single institution, so cross-site shift and sensor heterogeneity are not exercised. Nonetheless, the consistency of the geometric advantage across two distinct modalities, transform settings, and evaluation regimes (zero-shot and limited-support) suggests that the findings are not artifacts of a single favorable configuration. The construction itself extends to any product of a periodic manifold and a positive-amplitude manifold, which is the relevant setting for respiratory signals, gait cycles, and similar paired oscillatory translation problems.

By enabling waveform translation from convenient wearable sites to clinically validated locations, this work could expand access to continuous cardiovascular monitoring. However, clinical decisions should not be based solely on translated waveforms without independent validation, as silent translation failures could lead to missed or incorrect diagnoses. Additionally, physiological waveform data carry inherent re-identification risk and should be handled under appropriate data-protection protocols.

\section{Conclusion}
\label{sec:conclusion}

We introduced cylindrical geodesic flow matching, the first framework, to our knowledge, that performs paired cardiovascular waveform translation directly in instantaneous phase--amplitude coordinates. Embedding each sample on the cylinder $S^1 \times \mathbb{R}_{>0}$ gives a closed-form geodesic bridge with a constant-norm oracle field and an intrinsic tangent parameterization that keeps the learned dynamics on the manifold. We proved that this bridge removes the amplitude-shrinkage and frequency-warping artifacts that affine interpolation introduces between oscillatory endpoints and that it minimizes intrinsic length under the natural product metric. Empirically, the method matched or exceeded both interpolation baselines and direct supervised prediction on zero-shot PPG and limited-support SCG translation, reducing HT Dist, $L_2$, and DTW by up to ${\sim}15\%$ over the strongest competing baseline, and held its advantage under smaller adaptation budgets and coarser solvers. Together with the Hilbert naive baseline (which fails in the same coordinate system once the geometry is removed), these results point to bridge geometry as a key inductive bias for flow matching on physiological time series, and to periodic--positive product manifolds as a natural design space for paired translation of oscillatory signals.

\begin{ack}
Funding, IRB, and competing-interest disclosures will be added in the final version.
\end{ack}

\appendix

\bibliographystyle{plainnat}
\bibliography{references}

\section{Proofs of Theorems and Propositions}
\label{app:proofs}

\begin{theorem}[Smooth isometric cylindrical embedding and local representation equivalence]
\label{thm:embedding}
Let
\[
\zeta(\tau)=a(\tau)e^{i\phi(\tau)}\in\mathbb{C}\setminus\{0\}
\]
be a nonvanishing analytic representation of a waveform under a generic analytic-signal transform, so that its phase--amplitude coefficients define points
\[
([\phi(\tau)],a(\tau))\in S^1\times\mathbb{R}_{>0},
\qquad S^1\simeq \mathbb{R}/2\pi\mathbb{Z}.
\]
Define the cylindrical embedding
\[
\chi:S^1\times\mathbb{R}_{>0}\to\mathbb{R}^3,
\qquad
\chi([\phi],a):=(\cos\phi,\sin\phi,a),
\]
and equip \(S^1\times\mathbb{R}_{>0}\) with the product metric
\[
m=d\phi^2+da^2.
\]

Then:
\begin{enumerate}
\item \(\chi\) is a smooth isometric embedding of \(S^1\times\mathbb{R}_{>0}\) into \(\mathbb{R}^3\). In particular, it is injective, an immersion, and a homeomorphism onto its image.

\item Let \(I=[\alpha,\beta]\subset\mathbb{R}\) be any phase chart with \(\beta-\alpha<2\pi\), and let
\[
0<a_{\min}\le a_{\max}<\infty.
\]
Then the induced chart map
\[
\chi_I:I\times[a_{\min},a_{\max}]\to\mathbb{R}^3,
\qquad
\chi_I(\phi,a):=(\cos\phi,\sin\phi,a)
\]
is a diffeomorphism of manifolds with boundary onto its image and is bi-Lipschitz.

\item Consequently, on any bounded amplitude interval and any phase chart of width strictly less than \(2\pi\), modeling distributions over the coefficients \((\phi,a)\) is equivalent, up to a smooth bi-Lipschitz change of variables, to modeling distributions over the cylindrical embeddings
\[
g=\chi_I(\phi,a)\in\mathbb{R}^3.
\]
In particular, probability distributions and densities in one representation transfer to the other by the standard change-of-variables formula.
\end{enumerate}

\textit{Proof.}
For Item 1, under the standard realization
\[
S^1=\{(x,y)\in\mathbb{R}^2:x^2+y^2=1\},
\]
the image of \(\chi\) is the open cylinder
\[
\chi\bigl(S^1\times\mathbb{R}_{>0}\bigr)
=
\{(x,y,z)\in\mathbb{R}^3:x^2+y^2=1,\ z>0\},
\]
and the inverse is
\[
(x,y,z)\longmapsto \bigl((x,y),z\bigr).
\]
Thus \(\chi\) is injective and a homeomorphism onto its image. In a local angle coordinate \(\phi\),
\[
\partial_\phi \chi = (-\sin\phi,\cos\phi,0),
\qquad
\partial_a \chi = (0,0,1),
\]
so the Jacobian matrix is
\[
D\chi_I(\phi,a)=
\begin{pmatrix}
-\sin\phi & 0\\
\cos\phi & 0\\
0 & 1
\end{pmatrix},
\]
which has full column rank everywhere; hence \(\chi\) is an immersion. Moreover,
\[
\langle \partial_\phi\chi,\partial_\phi\chi\rangle = 1,
\qquad
\langle \partial_\phi\chi,\partial_a\chi\rangle = 0,
\qquad
\langle \partial_a\chi,\partial_a\chi\rangle = 1,
\]
so the pullback of the Euclidean metric is exactly
\[
\chi^*(dx^2+dy^2+dz^2)=d\phi^2+da^2=m.
\]
Therefore \(\chi\) is a smooth isometric embedding.

For Item 2, because \(\beta-\alpha<2\pi\), the arc map
\[
\phi\longmapsto (\cos\phi,\sin\phi)
\]
is injective on \(I\), and thus admits a smooth inverse branch \(\operatorname{Arg}_I\). Hence
\[
\chi_I^{-1}(x,y,z)=\bigl(\operatorname{Arg}_I(x+iy),z\bigr)
\]
is smooth on the image of \(\chi_I\), so \(\chi_I\) is a diffeomorphism onto its image.

Now let \(p_j=(\phi_j,a_j)\in I\times[a_{\min},a_{\max}]\) for \(j=1,2\). A direct computation gives
\[
\|\chi_I(p_1)-\chi_I(p_2)\|^2
=
4\sin^2\!\Bigl(\frac{\phi_1-\phi_2}{2}\Bigr)+(a_1-a_2)^2.
\]
Since \(|\phi_1-\phi_2|\le \beta-\alpha<2\pi\), there exists
\[
c_I:=\frac{2\sin((\beta-\alpha)/2)}{\beta-\alpha}>0
\]
such that
\[
c_I|\phi_1-\phi_2|
\le
2\Bigl|\sin\!\Bigl(\frac{\phi_1-\phi_2}{2}\Bigr)\Bigr|
\le
|\phi_1-\phi_2|.
\]
Substituting into the previous identity yields
\[
\min\{c_I,1\}^2\bigl(|\phi_1-\phi_2|^2+|a_1-a_2|^2\bigr)
\le
\|\chi_I(p_1)-\chi_I(p_2)\|^2
\le
|\phi_1-\phi_2|^2+|a_1-a_2|^2.
\]
Hence \(\chi_I\) is bi-Lipschitz.

Item 3 is immediate: on every bounded amplitude interval and every phase chart of width \(<2\pi\), the map \((\phi,a)\mapsto g=\chi_I(\phi,a)\) is a smooth bi-Lipschitz change of variables, so coefficient-space and embedding-space distributions are equivalent up to this coordinate transform. \hfill\(\square\)
\end{theorem}

\begin{proposition}[Cylindrical geodesic interpolation and oracle velocity]
\label{prop:geodesic}
Fix one sample \(s\in\{1,\dots,K\}\) from the uniformly sampled signals and define
\[
h_A(s)=([\phi_A(s)],a_A(s)),
\qquad
h_B(s)=([\phi_B(s)],a_B(s))\in S^1\times\mathbb{R}_{>0},
\]
where \(a_A(s),a_B(s)>0\). Let
\[
\Delta\phi(s):=\mathrm{wrap}\bigl(\phi_B(s)-\phi_A(s)\bigr)\in(-\pi,\pi],
\qquad
\Delta a(s):=a_B(s)-a_A(s),
\]
and equip \(S^1\times\mathbb{R}_{>0}\) with the product metric
\[
m=d\phi^2+da^2.
\]
For \(t\in[0,1]\), define
\[
\gamma_C(t;s):=\bigl([\phi_A(s)+t\,\Delta\phi(s)],\,a_A(s)+t\,\Delta a(s)\bigr)
\]
and its cylindrical embedding
\[
g_t(s):=\bigl(\cos\phi_t(s),\,\sin\phi_t(s),\,a_t(s)\bigr)\in\mathbb{R}^3,
\]
where
\[
\phi_t(s):=\phi_A(s)+t\,\Delta\phi(s),
\qquad
a_t(s):=a_A(s)+t\,\Delta a(s).
\]

Then:
\begin{enumerate}
\item \(\gamma_C(\cdot;s)\) is a constant-speed minimizing geodesic joining \(h_A(s)\) to \(h_B(s)\), with intrinsic speed and length
\[
\|\dot\gamma_C(t;s)\|_{m}
=
\sqrt{(\Delta\phi(s))^2+(\Delta a(s))^2},
\qquad
L_C(s)
=
\sqrt{(\Delta\phi(s))^2+(\Delta a(s))^2}.
\]
If \(|\Delta\phi(s)|<\pi\), this minimizing geodesic is unique.

\item The corresponding oracle velocity along the embedded path,
\[
v_t^\star(s):=\frac{d}{dt}g_t(s)
=
\begin{bmatrix}
-\sin\phi_t(s)\,\Delta\phi(s)\\[2pt]
\phantom{-}\cos\phi_t(s)\,\Delta\phi(s)\\[2pt]
\Delta a(s)
\end{bmatrix},
\]
has constant Euclidean norm
\[
\|v_t^\star(s)\|_2^2=(\Delta\phi(s))^2+(\Delta a(s))^2.
\]
\end{enumerate}

\textit{Proof.}
Under the product metric \(m=d\phi^2+da^2\), geodesics on \(S^1\times\mathbb{R}_{>0}\) are products of geodesics on the two factors. The phase component
\[
t\longmapsto [\phi_A(s)+t\,\Delta\phi(s)]
\]
is the minimizing geodesic on \(S^1\) associated with the principal wrapped increment \(\Delta\phi(s)\), and it is unique when \(|\Delta\phi(s)|<\pi\). Since \(\mathbb{R}_{>0}\) is convex, the amplitude component
\[
t\longmapsto a_A(s)+t\,\Delta a(s)
\]
lies in \(\mathbb{R}_{>0}\) for all \(t\in[0,1]\) and is the unique Euclidean geodesic joining \(a_A(s)\) to \(a_B(s)\). Hence \(\gamma_C(\cdot;s)\) is a constant-speed minimizing geodesic with velocity \((\Delta\phi(s),\Delta a(s))\), so
\[
\|\dot\gamma_C(t;s)\|_{m}
=
\sqrt{(\Delta\phi(s))^2+(\Delta a(s))^2}
\]
for all \(t\), and therefore
\[
L_C(s)
=
\sqrt{(\Delta\phi(s))^2+(\Delta a(s))^2}.
\]

Differentiating the embedded path gives the stated formula for \(v_t^\star(s)\). Using \(\sin^2\phi_t(s)+\cos^2\phi_t(s)=1\), we obtain
\[
\|v_t^\star(s)\|_2^2
=
\sin^2\phi_t(s)\,(\Delta\phi(s))^2
+
\cos^2\phi_t(s)\,(\Delta\phi(s))^2
+
(\Delta a(s))^2
=
(\Delta\phi(s))^2+(\Delta a(s))^2,
\]
which is independent of \(t\). \hfill\(\square\)
\end{proposition}

For the two comparison results below, we specialize Proposition~\ref{prop:geodesic} to a single sample and suppress the index \(s\). Thus
\[
\begin{aligned}
h_A&=([\phi_A],a_A),
&
h_B&=([\phi_B],a_B), \\
\Delta\phi&:=\mathrm{wrap}(\phi_B-\phi_A)\in(-\pi,\pi],
&
\Delta a&:=a_B-a_A,
\end{aligned}
\]
and \(\gamma_C\) denotes the corresponding cylindrical geodesic. When convenient, we identify \(\gamma_C\) with its complex realization
\[
\gamma_C(t)=(a_A+t\,\Delta a)e^{i(\phi_A+t\,\Delta\phi)}.
\]
For comparison, define the affine path in \(\mathbb{C}\) by
\[
w_A:=a_Ae^{i\phi_A},
\qquad
w_B:=a_Be^{i\phi_B},
\qquad
\gamma_{\mathrm{aff}}(t):=(1-t)w_A+t\,w_B=:w(t).
\]

\begin{theorem}[Amplitude and phase-rate fidelity of cylindrical vs.\ affine interpolation]
\label{thm:fidelity}
Using the single-sample notation above, assume \(a_A=a_B=a>0\), so that
\[
w_A=a e^{i\phi_A},
\qquad
w_B=a e^{i(\phi_A+\Delta\phi)},
\qquad
\gamma_C(t)=a\,e^{i(\phi_A+t\Delta\phi)}.
\]
For any nonvanishing path \(\gamma(t)=a_t e^{i\phi_t}\), define the amplitude distortion
\[
D_{\mathrm{amp}}(\gamma):=\int_0^1 (a_t-a)^2\,dt,
\]
and, whenever a continuous phase lift exists on \([0,1]\), define the interpolation-time phase-rate distortion
\[
D_{\mathrm{freq}}(\gamma):=\int_0^1 \bigl(\dot\phi_t-\Delta\phi\bigr)^2\,dt.
\]

Then:
\begin{enumerate}
\item For the cylindrical path \(\gamma_C\),
\[
D_{\mathrm{amp}}(\gamma_C)=0,
\qquad
D_{\mathrm{freq}}(\gamma_C)=0.
\]
Thus the amplitude remains exactly constant and the phase rate equals the desired increment \(\Delta\phi\) for all \(t\).

\item For the affine path \(\gamma_{\mathrm{aff}}\):
\begin{enumerate}
\item The amplitude is nonconstant whenever \(\Delta\phi\neq 0\), and
\[
a_t^{\mathrm{aff}}
=
a\sqrt{(1-t)^2+t^2+2t(1-t)\cos\Delta\phi}.
\]
Its minimum occurs at \(t=\tfrac12\), where
\[
a_{1/2}^{\mathrm{aff}}
=
a\,\bigl|\cos(\Delta\phi/2)\bigr|.
\]
Consequently,
\[
D_{\mathrm{amp}}(\gamma_{\mathrm{aff}})>0
\qquad\text{for all }\Delta\phi\neq 0.
\]
Moreover, for \(0<|\Delta\phi|<\pi\),
\begin{align}
D_{\mathrm{amp}}(\gamma_{\mathrm{aff}})
&=
a^{2}\!\left[
\frac{2+\cos\Delta\phi}{3}
-
\frac{1+\cos\Delta\phi}{\sqrt{2(1-\cos\Delta\phi)}}\,
\operatorname{asinh}\!\sqrt{\frac{1-\cos\Delta\phi}{1+\cos\Delta\phi}}
\right]
\label{eq:Damp-cos}
\\[4pt]
&=
a^{2}\!\left[
\frac{2+\cos\Delta\phi}{3}
-
\frac{\cos^{2}\!\left(\tfrac{\Delta\phi}{2}\right)}
{\bigl|\sin\!\left(\tfrac{\Delta\phi}{2}\right)\bigr|}\,
\operatorname{asinh}\!\Bigl(\bigl|\tan\!\left(\tfrac{\Delta\phi}{2}\right)\bigr|\Bigr)
\right],
\label{eq:Damp-halfangle}
\end{align}
with continuous extension
\[
D_{\mathrm{amp}}(\gamma_{\mathrm{aff}})=\frac13\,a^2
\qquad\text{at }|\Delta\phi|=\pi.
\]
As \(\Delta\phi\to 0\),
\[
D_{\mathrm{amp}}(\gamma_{\mathrm{aff}})
=
a^2\,\frac{(\Delta\phi)^4}{120}
+
O\!\bigl((\Delta\phi)^6\bigr).
\]

\item If \(0<|\Delta\phi|<\pi\), then the affine path remains in \(\mathbb{C}\setminus\{0\}\), its phase is well-defined, and its phase rate is
\[
\omega_t^{\mathrm{aff}}
=
\frac{\sin\Delta\phi}{(1-t)^2+t^2+2t(1-t)\cos\Delta\phi}.
\]
This function is nonconstant, so
\[
D_{\mathrm{freq}}(\gamma_{\mathrm{aff}})>0
\qquad\text{for all }0<|\Delta\phi|<\pi.
\]
Moreover,
\[
\omega_{1/2}^{\mathrm{aff}}
=
2\tan\!\Bigl(\frac{\Delta\phi}{2}\Bigr),
\]
so \(|\omega_{1/2}^{\mathrm{aff}}|\to\infty\) as \(|\Delta\phi|\uparrow\pi\), and
\[
D_{\mathrm{freq}}(\gamma_{\mathrm{aff}})
=
\frac{(\Delta\phi)^6}{180}
+
O\!\bigl((\Delta\phi)^8\bigr)
\qquad\text{as }\Delta\phi\to0.
\]
\end{enumerate}

\item In the boundary case \(|\Delta\phi|=\pi\), the affine midpoint amplitude vanishes,
\[
a_{1/2}^{\mathrm{aff}}=0,
\]
so the affine path leaves \(\mathbb{C}\setminus\{0\}\) and its phase rate is not defined on all of \([0,1]\).
\end{enumerate}

In particular, the cylindrical geodesic is exactly amplitude-faithful and phase-rate-faithful, whereas the affine path introduces strictly positive integrated amplitude distortion for every nonzero phase gap, and strictly positive integrated phase-rate distortion whenever the affine path stays away from the origin.

\textit{Proof.}
Under the standing equal-amplitude assumption,
\[
\gamma_C(t)=a\,e^{i(\phi_A+t\Delta\phi)},
\qquad
\gamma_{\mathrm{aff}}(t)
=
a e^{i\phi_A}\bigl[(1-t)+t\,e^{i\Delta\phi}\bigr].
\]

\paragraph{Cylindrical path.}
By construction, the cylindrical path has constant amplitude \(a_t\equiv a\) and phase \(\phi_t=\phi_A+t\Delta\phi\). Hence
\[
D_{\mathrm{amp}}(\gamma_C)=0,
\qquad
D_{\mathrm{freq}}(\gamma_C)=0.
\]

\paragraph{Affine-path amplitude.}
Write
\[
\gamma_{\mathrm{aff}}(t)=a e^{i\phi_A}\xi_t,
\qquad
\xi_t:=(1-t)+t\,e^{i\Delta\phi}.
\]
Therefore
\begin{equation}
a_t^{\mathrm{aff}}
=
a|\xi_t|
=
a\sqrt{(1-t)^2+t^2+2t(1-t)\cos\Delta\phi}.
\label{eq:env_aff}
\end{equation}
Using \(1-\cos\Delta\phi=2\sin^2(\Delta\phi/2)\), this can be rewritten as
\[
\left(\frac{a_t^{\mathrm{aff}}}{a}\right)^2
=
\cos^2\!\Bigl(\frac{\Delta\phi}{2}\Bigr)
+
4\Bigl(t-\frac12\Bigr)^2\sin^2\!\Bigl(\frac{\Delta\phi}{2}\Bigr).
\]
Hence the minimum occurs at \(t=\tfrac12\), giving
\begin{equation}
a_{1/2}^{\mathrm{aff}}
=
a\,\bigl|\cos(\Delta\phi/2)\bigr|.
\label{eq:env_min}
\end{equation}
As \(|\Delta\phi|\uparrow\pi\), this midpoint amplitude approaches zero, reflecting destructive self-interference; by contrast, the cylindrical path keeps \(a_t\equiv a\), so no spurious fading is introduced.
If \(\Delta\phi\neq 0\), then \(a_{1/2}^{\mathrm{aff}}<a\), so the integrand
\[
(a_t^{\mathrm{aff}}-a)^2
\]
is continuous and not identically zero; therefore
\[
D_{\mathrm{amp}}(\gamma_{\mathrm{aff}})>0.
\]

To obtain the closed form, set
\[
c:=\cos\Delta\phi,
\qquad
u(t):=2t(1-t).
\]
Then
\[
\frac{D_{\mathrm{amp}}(\gamma_{\mathrm{aff}})}{a^2}
=
\int_0^1\Bigl(\sqrt{1+(c-1)u(t)}-1\Bigr)^2dt.
\]
Since
\[
\Bigl(\sqrt{1+(c-1)u}-1\Bigr)^2
=
(c-1)u+2-2\sqrt{1+(c-1)u},
\]
we obtain
\[
\frac{D_{\mathrm{amp}}(\gamma_{\mathrm{aff}})}{a^2}
=
\frac{c-1}{3}+2-2I,
\qquad
I:=\int_0^1\sqrt{1+(c-1)u(t)}\,dt.
\]
Now
\[
u(t)=2t(1-t)=\frac12-2\Bigl(t-\frac12\Bigr)^2.
\]
Define
\[
A:=\frac{1+c}{2}=\cos^2\!\frac{\Delta\phi}{2},
\qquad
\kappa:=2(1-c)=4\sin^2\!\frac{\Delta\phi}{2}.
\]
With \(s:=t-\tfrac12\in[-\tfrac12,\tfrac12]\),
\[
I
=
\int_{-1/2}^{1/2}\sqrt{A+\kappa s^2}\,ds
=
2\int_0^{1/2}\sqrt{A+\kappa s^2}\,ds.
\]
Using
\[
\int \sqrt{A+\kappa s^2}\,ds
=
\frac{s}{2}\sqrt{A+\kappa s^2}
+
\frac{A}{2\sqrt{\kappa}}
\operatorname{asinh}\!\left(\frac{\sqrt{\kappa}\,s}{\sqrt{A}}\right),
\]
and \(A+\kappa/4=1\), we get
\[
I
=
\frac12
+
\frac{A}{\sqrt{\kappa}}
\operatorname{asinh}\!\left(\frac{\sqrt{\kappa}}{2\sqrt{A}}\right).
\]
Substituting back yields
\[
\frac{D_{\mathrm{amp}}(\gamma_{\mathrm{aff}})}{a^2}
=
\frac{2+c}{3}
-
\frac{2A}{\sqrt{\kappa}}
\operatorname{asinh}\!\left(\frac{\sqrt{\kappa}}{2\sqrt{A}}\right),
\]
which is exactly \eqref{eq:Damp-cos}--\eqref{eq:Damp-halfangle}. At \(|\Delta\phi|=\pi\), the continuous limit gives
\[
D_{\mathrm{amp}}(\gamma_{\mathrm{aff}})=\frac13\,a^2.
\]

For the small-angle expansion, let \(x:=\Delta\phi\). Using
\[
\cos x = 1-\frac{x^2}{2}+\frac{x^4}{24}+O(x^6),
\qquad
\sqrt{1+\varepsilon}=1+\frac{\varepsilon}{2}-\frac{\varepsilon^2}{8}+O(\varepsilon^3),
\]
with \(\varepsilon(t):=(\cos x-1)u(t)=O(x^2)\), we have
\[
\bigl(\sqrt{1+\varepsilon}-1\bigr)^2
=
\frac{\varepsilon^2}{4}+O(\varepsilon^3)
=
\frac{[(\cos x-1)u(t)]^2}{4}+O(x^6).
\]
Since
\[
(\cos x-1)^2=\frac{x^4}{4}+O(x^6),
\]
it follows that
\[
\frac{D_{\mathrm{amp}}(\gamma_{\mathrm{aff}})}{a^2}
=
\frac{x^4}{16}\int_0^1 u(t)^2\,dt + O(x^6).
\]
Because
\[
u(t)^2 = 4\bigl(t^2-2t^3+t^4\bigr),
\qquad
\int_0^1 u(t)^2\,dt = \frac{2}{15},
\]
we obtain
\[
D_{\mathrm{amp}}(\gamma_{\mathrm{aff}})
=
a^2\,\frac{x^4}{120}+O(x^6)
=
a^2\,\frac{(\Delta\phi)^4}{120}+O\!\bigl((\Delta\phi)^6\bigr).
\]

\paragraph{Affine-path phase rate.}
Assume now \(0<|\Delta\phi|<\pi\), so that \(\gamma_{\mathrm{aff}}(t)\neq 0\) for all \(t\in[0,1]\). Write
\[
\xi_t = R(t)+iI(t),
\qquad
R(t)=(1-t)+t\cos\Delta\phi,
\qquad
I(t)=t\sin\Delta\phi.
\]
For a nonvanishing complex path \(z(t)=R(t)+iI(t)\),
\[
\frac{d}{dt}\arg z(t)=\frac{R\,I'-I\,R'}{R^2+I^2}.
\]
Since
\[
R'(t)=\cos\Delta\phi-1,
\qquad
I'(t)=\sin\Delta\phi,
\]
we obtain
\begin{equation}
\omega_t^{\mathrm{aff}}
=
\frac{\sin\Delta\phi}{(1-t)^2+t^2+2t(1-t)\cos\Delta\phi}.
\label{eq:inst_freq_aff}
\end{equation}
Because the denominator depends nontrivially on \(t\) whenever \(\Delta\phi\neq 0\), the function \(\omega_t^{\mathrm{aff}}\) is nonconstant. On the other hand, the total phase change is exactly \(\Delta\phi\), so
\[
\int_0^1 \omega_t^{\mathrm{aff}}\,dt=\Delta\phi.
\]
Hence \(\Delta\phi\) is the mean value of \(\omega_t^{\mathrm{aff}}\) on \([0,1]\), and since \(\omega_t^{\mathrm{aff}}\) is nonconstant, it follows that
\[
D_{\mathrm{freq}}(\gamma_{\mathrm{aff}})>0.
\]

At \(t=\tfrac12\), the denominator in \eqref{eq:inst_freq_aff} is
\[
\frac12(1+\cos\Delta\phi)=\cos^2\!\Bigl(\frac{\Delta\phi}{2}\Bigr),
\]
so
\[
\omega_{1/2}^{\mathrm{aff}}
=
\frac{\sin\Delta\phi}{\cos^2(\Delta\phi/2)}
=
2\tan\!\Bigl(\frac{\Delta\phi}{2}\Bigr),
\]
whose magnitude diverges as \(|\Delta\phi|\uparrow\pi\).

For the small-angle expansion, let \(x:=\Delta\phi\) and \(u(t):=2t(1-t)\). Then
\[
(1-t)^2+t^2+2t(1-t)\cos x
=
1+u(t)(\cos x-1).
\]
Using
\[
\sin x=x-\frac{x^3}{6}+O(x^5),
\qquad
\cos x-1=-\frac{x^2}{2}+O(x^4),
\]
we expand
\[
\frac{1}{1+u(\cos x-1)}
=
1-u(\cos x-1)+O(x^4)
=
1+\frac{u\,x^2}{2}+O(x^4).
\]
Therefore
\[
\omega_t^{\mathrm{aff}}
=
\Bigl(x-\frac{x^3}{6}+O(x^5)\Bigr)
\Bigl(1+\frac{u\,x^2}{2}+O(x^4)\Bigr)
=
x+x^3\Bigl(\frac{u}{2}-\frac16\Bigr)+O(x^5).
\]
Since \(u(t)=2t(1-t)\),
\[
\omega_t^{\mathrm{aff}}
=
x+x^3\Bigl(t-t^2-\frac16\Bigr)+O(x^5).
\]
Thus
\[
\omega_t^{\mathrm{aff}}-\Delta\phi
=
x^3\Bigl(t-t^2-\frac16\Bigr)+O(x^5),
\]
and squaring gives
\[
\bigl(\omega_t^{\mathrm{aff}}-\Delta\phi\bigr)^2
=
x^6\Bigl(t-t^2-\frac16\Bigr)^2+O(x^8).
\]
Since
\[
\int_0^1\Bigl(t-t^2-\frac16\Bigr)^2\,dt=\frac1{180},
\]
we conclude that
\begin{equation}
D_{\mathrm{freq}}(\gamma_{\mathrm{aff}})
=
\frac{\Delta\phi^6}{180}+O(\Delta\phi^8).
\label{eq:Dfreq_result}
\end{equation}

Finally, if \(|\Delta\phi|=\pi\), then
\[
\gamma_{\mathrm{aff}}(1/2)=0,
\]
so the affine path leaves \(\mathbb{C}\setminus\{0\}\), and its phase rate is not defined on all of \([0,1]\). \hfill\(\square\)
\end{theorem}

\begin{theorem}[Intrinsic path-length optimality of the cylindrical geodesic]
\label{thm:length}
Using the same single-sample notation, consider two endpoints
\[
([\phi_A],a_A),
\qquad
([\phi_B],a_B),
\qquad
\Delta\phi:=\mathrm{wrap}(\phi_B-\phi_A)\in(-\pi,\pi],
\qquad
\Delta a:=a_B-a_A,
\]
and equip \(S^1\times\mathbb{R}_{>0}\) with the product metric
\[
m=d\phi^2+da^2.
\]

Let \(L_C\) and \(L_{\mathrm{aff}}\) denote the intrinsic lengths of \(\gamma_C\) and of the intrinsic realization of \(\gamma_{\mathrm{aff}}\), whenever the latter is well-defined:
\[
L_C
=
\int_0^1 \sqrt{\dot\phi_C^2+\dot a_C^2}\,dt,
\qquad
L_{\mathrm{aff}}
=
\int_0^1 \sqrt{\dot\delta(t)^2+\dot\alpha(t)^2}\,dt.
\]

Then:
\begin{enumerate}
\item The cylindrical path has constant speed and minimal intrinsic length
\[
L_C=\sqrt{(\Delta\phi)^2+(\Delta a)^2}.
\]

\item If the affine chord avoids the origin on \([0,1]\), then
\[
L_{\mathrm{aff}}\ge L_C,
\]
with equality if and only if \(\Delta\phi=0\), i.e., the endpoints lie on the same complex ray and only the amplitude changes.

\item In the special case \(a_A=a_B=a\) and \(0<|\Delta\phi|<\pi\),
\[
\alpha(t)
=
a\sqrt{(1-t)^2+t^2+2t(1-t)\cos\Delta\phi},
\]
so
\[
\alpha(1/2)
=
a\bigl|\cos(\Delta\phi/2)\bigr|<a,
\]
and therefore
\[
L_{\mathrm{aff}}>L_C=|\Delta\phi|.
\]

\item In the boundary case \(|\Delta\phi|=\pi\), the affine chord passes through the origin at
\[
t_\ast=\frac{a_A}{a_A+a_B},
\]
so the induced intrinsic path is not defined on all of \([0,1]\), whereas
\[
L_C=\sqrt{\pi^2+(\Delta a)^2}
\]
remains finite.
\end{enumerate}

\textit{Proof.}
We measure lengths on \(S^1\times\mathbb{R}_{>0}\) with the product metric
\[
ds^2=d\phi^2+da^2.
\]
Thus, for any absolutely continuous curve \(\gamma(t)=(\phi(t),a(t))\),
\[
\|\dot\gamma(t)\|
=
\sqrt{\dot\phi(t)^2+\dot a(t)^2},
\qquad
\mathrm{Length}(\gamma)
=
\int_0^1 \|\dot\gamma(t)\|\,dt.
\]

For the cylindrical path,
\[
\phi_C(t)=\phi_A+t\,\Delta\phi,
\qquad
a_C(t)=a_A+t\,\Delta a,
\]
so
\[
\dot\phi_C=\Delta\phi,
\qquad
\dot a_C=\Delta a,
\]
and hence
\[
\|\dot\gamma_C\|
=
\sqrt{(\Delta\phi)^2+(\Delta a)^2}.
\]
Therefore
\[
L_C
=
\int_0^1 \|\dot\gamma_C\|\,dt
=
\sqrt{(\Delta\phi)^2+(\Delta a)^2},
\]
which is exactly the geodesic distance between the endpoints. This proves Item 1.

Now assume the affine chord avoids the origin on \([0,1]\). Then it admits an intrinsic realization
\[
\gamma_{\mathrm{aff}}(t)=(\delta(t),\alpha(t)),
\qquad
\delta(0)=\phi_A,
\qquad
\delta(1)=\phi_A+\Delta\phi,
\qquad
\alpha(t)=|w(t)|.
\]
Viewing \(\gamma_{\mathrm{aff}}\) as a curve in the Euclidean \((\phi,a)\)-plane, the triangle inequality for curves gives
\[
L_{\mathrm{aff}}
=
\int_0^1 |\dot\gamma_{\mathrm{aff}}(t)|\,dt
\ge
\left|\int_0^1 \dot\gamma_{\mathrm{aff}}(t)\,dt\right|
=
|\gamma_{\mathrm{aff}}(1)-\gamma_{\mathrm{aff}}(0)|
=
\sqrt{(\Delta\phi)^2+(\Delta a)^2}
=
L_C.
\]
Hence \(L_{\mathrm{aff}}\ge L_C\).

Equality in the triangle inequality holds only if the image of \(\gamma_{\mathrm{aff}}\) in the \((\phi,a)\)-plane is the straight segment joining its endpoints, up to monotone reparametrization. In that case, its complex realization would have the form
\[
z(\lambda)=\bigl(a_A+\lambda\,\Delta a\bigr)e^{i(\phi_A+\lambda\,\Delta\phi)},
\qquad
\lambda\in[0,1].
\]
Differentiating,
\[
z'(\lambda)
=
e^{i(\phi_A+\lambda\,\Delta\phi)}
\Bigl(\Delta a+i\,\Delta\phi\,(a_A+\lambda\,\Delta a)\Bigr),
\]
\[
z''(\lambda)
=
e^{i(\phi_A+\lambda\,\Delta\phi)}
\Bigl(2i\,\Delta a\,\Delta\phi-(\Delta\phi)^2(a_A+\lambda\,\Delta a)\Bigr).
\]
A direct computation yields
\[
\operatorname{Im}\!\bigl(\overline{z'(\lambda)}\,z''(\lambda)\bigr)
=
\Delta\phi\Bigl(2(\Delta a)^2+(\Delta\phi)^2(a_A+\lambda\,\Delta a)^2\Bigr).
\]
If \(\Delta\phi\neq 0\), this quantity is nonzero for every \(\lambda\in[0,1]\), so the image \(z([0,1])\) is not a Euclidean line segment in \(\mathbb{C}\). But the affine path \(w(t)=(1-t)w_A+t\,w_B\) is, by definition, a Euclidean line segment. Hence equality is impossible when \(\Delta\phi\neq 0\). If \(\Delta\phi=0\), then
\[
w(t)=e^{i\phi_A}(a_A+t\,\Delta a)
\]
stays on the same complex ray, so
\[
\gamma_{\mathrm{aff}}(t)=(\phi_A,a_A+t\,\Delta a),
\qquad
L_{\mathrm{aff}}=|\Delta a|=L_C.
\]
This proves Item 2.

For Item 3, when \(a_A=a_B=a\),
\[
\alpha(t)=|w(t)|
=
a\sqrt{(1-t)^2+t^2+2t(1-t)\cos\Delta\phi},
\]
so
\[
\alpha(1/2)=a\bigl|\cos(\Delta\phi/2)\bigr|<a
\qquad\text{for}\qquad 0<|\Delta\phi|<\pi.
\]
Thus the affine intrinsic path is not the geodesic
\[
t\longmapsto ([\phi_A+t\,\Delta\phi],a),
\]
and Item 2 implies
\[
L_{\mathrm{aff}}>L_C=|\Delta\phi|.
\]

For Item 4, if \(|\Delta\phi|=\pi\), then
\[
w_B=-a_B e^{i\phi_A}
\]
and therefore
\[
w(t)=e^{i\phi_A}\bigl((1-t)a_A-t\,a_B\bigr),
\]
which vanishes at
\[
t_\ast=\frac{a_A}{a_A+a_B}\in(0,1).
\]
Hence the affine path leaves \(\mathbb{C}\setminus\{0\}\) and does not define a curve in \(S^1\times\mathbb{R}_{>0}\) on all of \([0,1]\). The cylindrical path remains well-defined and has length
\[
L_C=\sqrt{\pi^2+(\Delta a)^2}.
\]
\hfill\(\square\)
\end{theorem}

\section{Supplementary Experimental Details}
\label{app:implementation}

\subsection{Dataset Details and Preprocessing}
\label{app:dataset-details}

Here, we describe the acquisition hardware, cohort composition, synchronization, segmentation, artifact rejection, and normalization used to construct the PPG and SCG translation datasets.

\paragraph{Hardware.}
Both datasets were collected with the Cardiotag wearable patch~\cite{chan2021enabling} and a synchronized Biopac MP150 system. The Cardiotag is a compact patch that integrates single-lead electrocardiogram (ECG), tri-axial seismocardiogram (SCG), and multi-wavelength photoplethysmography (PPG). The ECG analog front end is the ADS1291 (Texas Instruments, Dallas, TX, USA); the PPG analog front end is the MAX86170 (Maxim Integrated, San Jose, CA, USA), driving an SFH7016 multi-LED chip (OSRAM, Munich, Germany) and two VEMD8080 photodiodes (Vishay Semiconductors, Heilbronn, Germany). Each photodiode channel (PD1, PD2) supports green ($526\,\mathrm{nm}$), red ($660\,\mathrm{nm}$), and infrared ($950\,\mathrm{nm}$) illumination; in this work we use only the green channel of each diode. From the tri-axial SCG accelerometer we use the dorsoventral axis, which carries most of the cardiac mechanical energy at the chest. Patch ECG and PPG are sampled at $500\,\mathrm{Hz}$ and $67\,\mathrm{Hz}$, respectively. Biopac ECG and PPG are sampled at $2000\,\mathrm{Hz}$; finger PPG is acquired with a Nonin 8000AA-1 clip (Nonin, Plymouth, MN, USA) using a $663\,\mathrm{nm}$ red light source.

\paragraph{Cohort and protocols.}
The PPG cohort consists of 20 subjects, each recorded in a single session. The SCG cohort consists of 13 subjects and 24 total sessions, with most participants completing two recordings on separate days. Both datasets were acquired at the authors’ institution under separate IRB-approved protocols, and all participants provided written informed consent and were compensated for their time.

Each SCG session followed a standardized protocol comprising a 5-minute resting baseline followed by three stressor blocks. Each block included a 2-minute challenge, a 5-minute recovery recording, and a 2-minute survey. The challenges were administered in a fixed order: mental arithmetic, cold pressor, and sub-maximal exercise, resulting in a total session duration of approximately 32 minutes. For modeling, we retained only the four measurement phases: baseline, mental-arithmetic recovery, cold-pressor recovery, and exercise recovery. Challenge and survey intervals were excluded.

Each PPG session began with a 5-minute baseline recording, followed by alternating respiratory perturbations and recovery periods. The perturbations consisted of 1-minute deep breathing, resistive breathing, standing deep breathing, and standing straw breathing tasks, each followed by a 2-minute recovery measurement. The protocol also included 3-minute postural transitions—sit-to-stand, stand-to-sit, sit-to-supine, supine-to-legs-raised, legs-raised-to-standing, standing-to-supine, and supine-to-sitting—which we treat as natural perturbations of pulse morphology rather than explicit stressors.

\paragraph{Signal cleaning.}
Patch ECG sampled natively at \(500\,\mathrm{Hz}\) is retained at its original rate, patch PPG sampled at \(67\,\mathrm{Hz}\) is upsampled, and Biopac ECG and PPG sampled at \(2000\,\mathrm{Hz}\) are downsampled. Before resampling, each signal is denoised with a zero-phase fifth-order Butterworth bandpass filter using modality-specific passbands: \([1,30]\,\mathrm{Hz}\) for ECG, \([1,4]\,\mathrm{Hz}\) for PPG, and \([2,39]\,\mathrm{Hz}\) for SCG. These cutoff frequencies lie well below the \(250\,\mathrm{Hz}\) Nyquist frequency of the common grid and are chosen to retain the dominant cardiac morphology while attenuating baseline wander, out-of-band respiratory variation, and high-frequency sensor noise. 

\paragraph{R-peak detection.}
R-peaks are detected independently from the ECG signal of each device. In the paired-device dataset, the device-specific R-peak series are used to align recordings across devices. In the SCG dataset, R-peaks provide the cardiac-cycle anchors used to select SCG samples.

Detection uses a NeuroKit2-anchored consensus procedure for each ECG channel~\cite{makowski2021neurokit2}. The default NeuroKit2 detector first produces an initial candidate sequence, which is then refined using detections from Martínez \emph{et al.}\ ~\cite{martinez2004wavelet} and Kalidas and Tamil~\cite{kalidas2017real}. Within a sliding 100-beat window with stride 10, each NeuroKit2 RR interval is compared with the local median RR interval. Intervals shorter than \(0.7\times\) the local median are treated as candidate extra detections and redundant peaks are removed. Intervals longer than \(1.4\times\) the local median are treated as candidate missed beats; in these cases, a peak from either reference detector is inserted if it lies within \(\pm 50\,\mathrm{ms}\) of the expected beat location. Inserted peaks are sorted with the original sequence, and duplicate detections are removed. Thus, the final R-peak sequence for each device is not the unanimous intersection of the three detectors, but a NeuroKit2-based sequence corrected for missed and spurious beats using either auxiliary detector as supporting evidence.

\paragraph{Cross-device synchronization.}
The Cardiotag patches and the Biopac system run on independent clocks; moreover, when multiple Cardiotag patches are used, the patches themselves are not sample-synchronous. Consequently, sample indices across any pair of streams do not in general correspond to the same wall-clock instants, and all streams must be temporally aligned before cross-device modeling. We perform this alignment pairwise in two stages. First, we estimate a session-level lag from heart-rate variability: per-beat RR intervals from each device’s ECG are interpolated onto a common \(4\,\mathrm{Hz}\) grid, and the resulting tachograms are cross-correlated using an FFT-based implementation. Because RR-interval traces are largely device-independent and provide a distinctive per-session physiological signature, this step recovers the relative offset between each device pair to sub-second resolution without requiring waveform-level alignment. Second, we refine the alignment to sample resolution using sliding-window matched filtering on the bandpass-filtered ECG waveforms. In \(300\)-second windows with a \(20\)-second stride, ECG pairs are cross-correlated with the search constrained to a small neighborhood of the previous estimate, and per-window outliers are rejected using an IQR criterion. This refinement is applied continuously throughout the recording rather than only once at session onset: a fresh lag estimate is produced for every window, and each device timeline is shifted accordingly relative to a chosen reference stream. The resulting time-varying alignment compensates for the slow clock drift that accumulates among independently running devices over a \(\sim\)30-minute session, which would not be corrected by a single rigid offset. After temporal synchronization, all channels are represented on a common \(500\,\mathrm{Hz}\) grid. 

\paragraph{Beat and window extraction.}
For the SCG task, we use the R-peaks first to define a short reference segment for artifact rejection. Specifically, letting \(R\) denote the sample index of an R-peak, we extract a 350-sample segment, corresponding to \(0.7\,\mathrm{s}\) at \(500\,\mathrm{Hz}\), beginning at the R-peak. This reference segment is used only for quality control and is not used as input to the model. Within this segment, we compute the variance of the absolute deviation of the SCG signal. Beats are grouped by experimental phase, and within each phase we discard beats whose variance lies outside the 5th--95th percentile range of the phase-specific distribution. Because SCG is recorded at both the Sternum and Clavicle, a beat is retained only if it passes this criterion for both channels. Phase-specific thresholding is essential because exercise-recovery beats have intrinsically higher SCG energy than baseline beats; applying a single global threshold would therefore reject a disproportionate fraction of valid exercise data. After this filtering step, \(30{,}494\) SCG beats remain.

After artifact rejection, we construct the SCG model inputs. Although the rejection criterion is computed on the 350-sample reference segment, the actual beat provided to the model is a longer 1024-sample window, corresponding to \(2.048\,\mathrm{s}\). This window is centered at the midpoint of the reference segment:
\[
c_{\mathrm{beat}} = R + 175,
\qquad
\mathrm{window} = [c_{\mathrm{beat}}-512,\ c_{\mathrm{beat}}+512).
\]
As a result, the model input is intentionally asymmetric with respect to the R-peak, extending approximately \(674\,\mathrm{ms}\) before and \(1374\,\mathrm{ms}\) after it. This placement captures the AO/AC-bearing portion of the SCG waveform while retaining pre-R context, including the P-wave and the tail of the preceding beat, as well as the post-AC decay. The extended context also provides padding that helps suppress boundary artifacts from FFT-based analytic-signal computation.

For the PPG task, we segment the signal into fixed-length windows of 4096 samples, corresponding to \(8.192\,\mathrm{s}\), with a stride of 2048 samples, yielding 50\% overlap between adjacent windows. This produces \(11{,}483\) PPG windows. We do not apply the SCG variance-based rejection procedure to PPG, as the modality is substantially cleaner and the longer window duration reduces the influence of short-lived motion artifacts.

\paragraph{Normalization.}
Each beat or window is z-score normalized per channel along the time axis using the numerically-stabilized estimator $(x-\mu)/(\sigma+\varepsilon)$ with $\varepsilon=10^{-12}$. Conditioning channels are normalized with a NaN-aware variant so that occasional dropouts in auxiliary signals do not contaminate the statistics of the primary channels.

\subsection{Exact Split Definitions and Leakage Safeguards}
\label{app:split-details}

\paragraph{Subject-disjoint repeated splits.}
Both datasets are partitioned at the subject level, not the session level: all sessions from a given subject are assigned to exactly one of the train, validation, or test splits. Thus, for the SCG task, the held-out unit is a subject, including both Day~1 and Day~2 sessions when available, rather than an individual recording session. We repeat this subject-level partitioning over four random seeds, $\{6,10,11,13\}$, producing four independent train/validation/test assignments. For PPG, the split ratio is 50\%/25\%/25\%, corresponding to 10/5/5 subjects. For SCG, the same ratio gives approximately 7/3/3 subjects. The validation and test subjects for each dataset and random seed are listed in Table~\ref{tab:subject-splits}; the training set is the complement of the listed subjects. All reported metrics are summarized as means and standard deviations across these four subject-disjoint splits.

\begin{table*}[b]
\centering
\small
\caption{
Subject-disjoint validation and test assignments used for repeated evaluation. 
Each row corresponds to one random seed; the training set is the complement of the listed validation and test subjects. 
All sessions from a subject are assigned to the same split, ensuring no subject overlap across train, validation, and test sets.
}
\label{tab:subject-splits}
\begin{tabular}{cccc}
\toprule
Dataset & Seed & Validation subjects & Test subjects \\
\midrule
\multirow{4}{*}{PPG}
& 6
& S004, S006, S013, S015, S016
& S005, S007, S014, S018, S019 \\
& 10
& S003, S010, S011, S014, S018
& S006, S007, S013, S016, S024 \\
& 11
& S002, S004, S011, S015, S018
& S001, S005, S006, S013, S017 \\
& 13
& S002, S014, S016, S019, S020
& S001, S007, S017, S018, S024 \\
\midrule
\multirow{4}{*}{SCG}
& 6
& S005, S012, S017
& S004, S007, S009 \\
& 10
& S003, S004, S017
& S006, S007, S012 \\
& 11
& S006, S009, S015
& S002, S004, S007 \\
& 13
& S003, S007, S018
& S002, S011, S012 \\
\bottomrule
\end{tabular}
\end{table*}

\paragraph{Validation set role.}
The validation split is held subject-disjoint from both the train and test splits and is used exclusively for model selection. During training, we periodically evaluate the model on the validation split and retain the checkpoint with the best validation metric; this single checkpoint is then evaluated once on the held-out test split to produce the reported numbers. No test sample is observed at any point during training, hyperparameter selection, or checkpoint selection.

\paragraph{Causal support-set construction (SCG).}
For the few-shot adaptation experiments on SCG, the support set for each held-out test session is constructed as a strictly causal chronological prefix of that session's beats: the first $S$ beats in temporal order are designated as support and the remaining beats form the query set, where $S$ is specified in App.~\ref{app:support-ablation}. Because beats are gated to the four measurement phases and stored in chronological order within a session, this prefix spans the resting baseline and, for sessions whose baseline yields fewer than $S$ retained beats after artifact rejection, extends into the subsequent mental-arithmetic recovery phase. No randomization or stratification is applied within the session; the support is exactly the temporally earliest portion of the recording, mirroring the deployment setting in which a short enrollment segment is acquired before the model is asked to track the subject through subsequent activity.

\paragraph{Leakage safeguards.}
Our construction prevents data leakage across splits and between support and query sets in three key ways:
\begin{enumerate}
    \renewcommand{\labelenumi}{(\roman{enumi})}
    \item \emph{Subject disjointness.} Every session of a subject is assigned to exactly one split, so neither the trained weights nor the adaptation support has ever observed a test subject's morphology.

    \item \emph{Per-window normalization.} The z-score normalization is computed per window or per beat using only that sample's own statistics, never aggregated across sessions or splits, so no global moment can carry test information into training.

    \item \emph{Causal supports.} Because the SCG adaptation support is the chronological prefix of a held-out session and the query is its strict suffix, the model cannot adapt on a sample temporally posterior to one it is later evaluated on, eliminating the within-session leakage that random support sampling would introduce.
\end{enumerate}

\subsection{Analytic Transforms and Metric Definitions}
\label{app:metric-details}
Both instantiations of the generic analytic transform $\mathcal{A}$ share a common FFT-based core. Given a real signal $x \in \mathbb{R}^T$, the standard Hilbert transform (HT) constructs the analytic signal
\[
z = \mathcal{F}^{-1}\!\bigl(H \odot \mathcal{F}(x)\bigr),
\]
where $\mathcal{F}$ denotes the DFT along the time axis and $H \in \mathbb{R}^T$ is the one-sided Hilbert filter with $H_0 = H_{T/2} = 1$, $H_k = 2$ for $1 \le k < T/2$, and $H_k = 0$ otherwise (even $T$; the odd-$T$ case is analogous). The instantaneous amplitude envelope is $a = |z|$ (clamped below at $\varepsilon = 10^{-12}$ to avoid division by zero) and the instantaneous phase is $\phi = \mathrm{unwrap}(\arg z)$, where the unwrap operator removes $2\pi$ discontinuities by cumulative correction of successive phase differences. A scalar mean angular frequency is estimated as $\hat{\omega} = (\phi_{T} - \phi_1) / ((T-1)\Delta t)$.

The extended Hilbert transform (EHT)~\cite{matsuki2023extended} refines the HT phase estimate without performing a full empirical mode decomposition. Starting from the HT quantities above, it (i) detrends the phase to obtain $u_H = \phi_H - \hat{\omega}\,t$, (ii) takes the DFT of $u_H$ and extracts the positive-frequency half-spectrum, (iii) estimates the dominant IMF order $m = \mathrm{round}(\hat{\omega}\,T\,\Delta t / 2\pi)$ and applies a piecewise spectral correction that reconstructs the positive-frequency component, (iv) inverts back to the time domain and smooths the detrended phase with a moving-average filter of span $\max(5, \lfloor 0.02\,T \rfloor)$, (v) flags outliers via median absolute deviation (MAD threshold 3.5) and replaces them by linear interpolation, and (vi) clamps per-sample phase velocity to $\pm 0.45\;\mathrm{rad/sample}$ to suppress residual jumps. The refined phase is $\hat{\phi} = \mathrm{unwrap}(\hat{\omega}\,t + u_{\mathrm{smooth}})$. The amplitude envelope is the same as the standard HT.

\paragraph{Center cropping.}
Because the DFT treats its input as periodic, both HT and EHT can produce boundary artifacts near the first and last samples of a finite window. To mitigate this, the analytic transform is computed on a window that is intentionally wider than the segment ultimately used by the model. For the PPG task, the dataset window is 4096 samples ($8.192\,\mathrm{s}$) while the model operates on the center 2048 samples ($4.096\,\mathrm{s}$); after computing HT or EHT on the full 4096-sample window, both amplitude and phase are center-cropped to 2048 samples, discarding 1024 samples from each end. The same strategy is applied to SCG: the transform is computed on the full 1024-sample beat window, which extends well beyond the physiologically active portion of the beat (Section~\ref{app:dataset-details}), and the resulting amplitude and phase are then center-cropped to the model size, so that the retained segment is insulated from boundary effects on both sides.

\paragraph{Hilbert transform distance.}
HT Dist measures the mean per-beat intrinsic distance in the analytic-signal domain. For a predicted and target signal pair, the standard Hilbert transform is applied to both to obtain amplitude--phase pairs $(a_{\mathrm{pred}}, \phi_{\mathrm{pred}})$ and $(a_{\mathrm{tgt}}, \phi_{\mathrm{tgt}})$. The pointwise intrinsic distance is
\[
d(s) = \sqrt{\bigl(a_{\mathrm{pred}}(s) - a_{\mathrm{tgt}}(s)\bigr)^2 + \bigl(\mathrm{circ}(\phi_{\mathrm{pred}}(s),\,\phi_{\mathrm{tgt}}(s))\bigr)^2},
\]
where $\mathrm{circ}(\phi_1, \phi_2) = \bigl|(\phi_1 - \phi_2 + \pi) \bmod 2\pi - \pi\bigr|$ maps the signed phase difference to the minimum circular distance in $[0, \pi]$. This is precisely the product metric on $S^1 \times \mathbb{R}_{>0}$ introduced in Section~\ref{sec:background}, evaluated pointwise between the predicted and target analytic representations. HT Dist is the sample-averaged intrinsic distance per beat, further averaged across the batch. Because the metric operates in the analytic domain, it penalizes phase distortions that alter waveform morphology even when they leave the time-domain amplitude largely unchanged, complementing purely time-domain criteria.

\paragraph{Root-mean-square error.}
$L_2$ is the per-beat root-mean-square error between predicted and target waveforms in the time domain, averaged across the test batch. Both signals are channel-wise z-score normalized prior to comparison as described in Section~\ref{app:dataset-details}.

\paragraph{Dynamic time warping distance.}
Dynamic time warping (DTW) \cite{sakoe1978dynamic} measures the cost of an optimal nonlinear alignment between two sequences, thereby quantifying waveform dissimilarity that is tolerant to local temporal shifts and stretches. Unlike $L_2$, which assumes strict sample-to-sample correspondence, DTW finds a monotonic warping path through the pairwise cost matrix that minimizes the cumulative Euclidean distance, making it sensitive to shape differences that point-wise metrics may underestimate when slight timing variations are present \cite{berndt1994using,keogh2005exact}. Because the quadratic complexity of the standard algorithm is prohibitive at full sampling rate, we first reduce the temporal resolution by a factor of four using non-overlapping mean pooling before computing per-beat DTW distances. The reported values are the mean and standard deviation across the test batch.

\subsection{Model and Training Details}
\label{app:model-details}

\paragraph{Backbone architecture.}
All learned methods use the same encoder-decoder architecture with skip connections \cite{ronneberger2015unet}, adapted to one-dimensional sequences. Each network comprises convolutional downsampling and upsampling stages with residual connections across three or four resolution levels, a bottleneck module at the coarsest scale, sinusoidal time embeddings \cite{ho2020ddpm} for the scalar flow time $t$, and $5\%$ dropout on the concatenated input. The supervised predictor uses the same architecture, receiving only the source waveform and ECG conditioning and producing a single-pass endpoint prediction.

The four methods differ in their state representation, which determines input and output dimensionality. All variants use four resolution levels with per-level feature widths $(32,32,64,64)$. The time-series affine model receives three channels: the current state, the source waveform $x_0$, and the ECG, and produces one output channel. The supervised model receives two channels: the source waveform $x_0$ and the ECG, and produces one output channel. Hilbert-naive models receive the amplitude--phase representation of all three signals (six channels total) and produce two output channels. Cylindrical geodesic models receive the cylindrical embedding of all three signals (nine channels total) and produce two intrinsic output channels ($\omega_\theta$, $\dot{a}_\theta$) mapped to the ambient tangent vector $(-\sin\phi\,\omega_\theta,\cos\phi\,\omega_\theta,\dot{a}_\theta)$ before loss evaluation. All variants have approximately $1.0\mathrm{M}$ trainable parameters; differences across methods are under $0.2\%$ and arise solely from the varying input and output channel counts.

\paragraph{Training objective and optimization.}
Flow-matching models draw $t\sim\mathrm{Uniform}(0,1)$ independently per sample and minimize the mean-squared error between the predicted and oracle velocity fields (Eq.~\eqref{eq:fm-loss}); the supervised model minimizes the mean-squared reconstruction error at the endpoint. All models are optimized with AdamW \cite{loshchilov2019adamw} using $\beta=(0.9,0.99)$, weight decay $10^{-2}$, and gradient-norm clipping at $1.0$. The learning-rate schedule applies linear warmup over the first $1\%$ of optimizer steps (starting at $1\%$ of the peak rate), followed by cosine decay to $10^{-3}$ times the peak rate. To prevent the model from collapsing to the source during conditioning, we perturb the $x_0$ input with additive Gaussian noise ($\sigma=0.05$) during training only. Checkpoints are evaluated against held-out validation subjects; the final model is selected by validation HT Dist.

\paragraph{Task-specific settings.}
Table~\ref{tab:training-hparams} summarizes the per-task hyperparameters. For PPG zero-shot transfer, the analytic transform is applied to $4096$-sample windows and the model operates on the center $2048$ samples (App.~\ref{app:metric-details}). For SCG pretraining, the transform is applied to $1024$-sample beat windows and the model operates on the center $512$ samples. Both tasks use the split seeds $\{6,10,11,13\}$ defined in Table~\ref{tab:subject-splits}.

\begin{table}[h]
    \centering
    \small
    \caption{Per-task training hyperparameters. PPG models are evaluated zero-shot; SCG models are pretrained then adapted per held-out session. Pretraining uses the validation-selected checkpoint; adaptation runs to completion and uses the final checkpoint (no validation during adaptation).}
    \label{tab:training-hparams}
    \begin{tabular}{lcccc}
        \toprule
        Task & Batch size & Peak LR & Epochs & Val.\ every \\
        \midrule
        PPG zero-shot & 512 & $3\times10^{-4}$ & 20{,}000 & 5 \\
        SCG pretraining & 4096 & $3\times10^{-4}$ & 20{,}000 & 5 \\
        SCG adaptation & 128 & $1\times10^{-4}$ & 50 & --- \\
        \bottomrule
    \end{tabular}
\end{table}

\paragraph{Compute environment.}
Models were trained on two NVIDIA RTX~5090 GPUs (32\,GB VRAM each; peak allocation 7.7\,GB per GPU) with CUDA~13.0, PyTorch~2.8, and Python~3.13. Approximate wall-clock training times per run are: supervised predictor $\approx$50\,min; all three flow-matching models (time-series affine, Hilbert-naive, and cylindrical geodesic) each $\approx$4.3\,h.

\subsection{Support-Size Ablation}
\label{app:support-ablation}
We probe sensitivity to the size of the per-session support set used during SCG limited-support adaptation. For each held-out session, we vary the support budget over $\{16, 32, 64, 128, 256, 512\}$ beats drawn from the temporally earliest portion of the recording and re-evaluate the trained cylindrical geodesic model on the chronological remainder of the session, leaving architecture, optimizer, and training schedule fixed.

\begin{figure}[h]
  \centering
  \includegraphics[width=\linewidth]{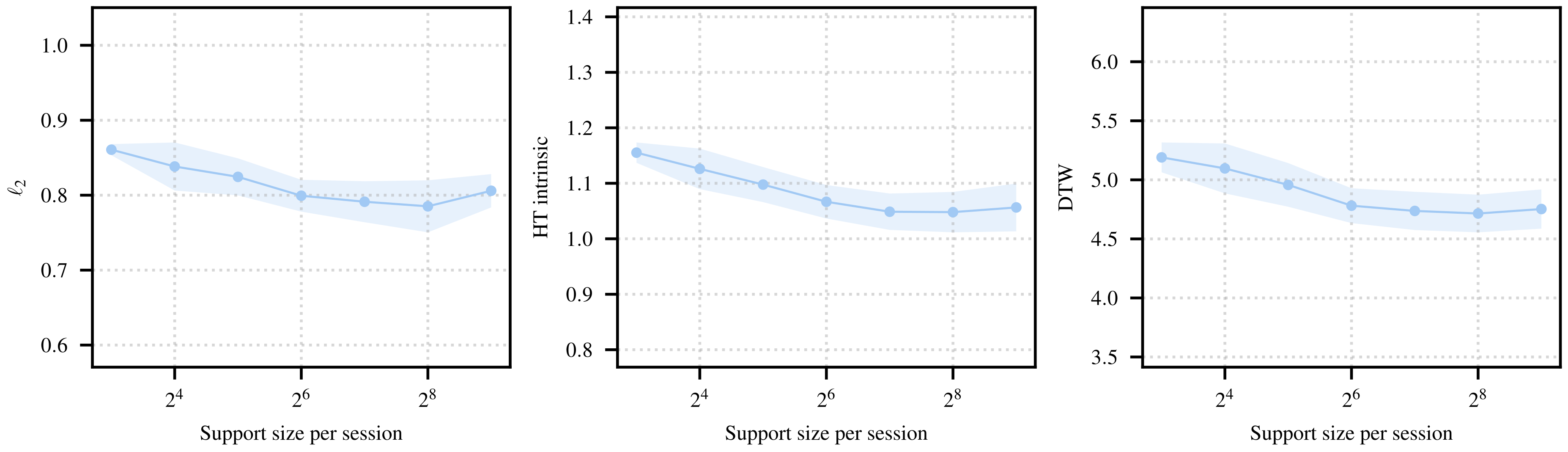}
  \caption{Support-size ablation on SCG. Mean and $\pm 1$ standard deviation of $L_2$, HT intrinsic distance, and DTW are plotted against the support-set size per held-out session, aggregated across subjects and seeds.}
  \label{fig:support-ablation}
\end{figure}

Figure~\ref{fig:support-ablation} shows that the steepest improvements occur between $16$ and $128$ support beats and that gains beyond $256$ beats are within the reported variability. The 512-beat operating point used in the main text therefore lies in a diminishing-returns regime, so the headline numbers are not driven by an unusually generous support budget; conversely, useful adaptation is retained at $32$--$64$ support beats, indicating that the cylindrical geodesic path does not rely on long calibration windows.

\subsection{Solver-Step Ablation}
\label{app:solver-ablation}
We examine sensitivity to the number of steps used by the projected Heun solver at inference time. Fixing the trained cylindrical geodesic model, we re-evaluate it with solver budgets in $\{4, 8, 16, 32, 64, 100\}$ steps.

\begin{figure}[h]
  \centering
  \includegraphics[width=\linewidth]{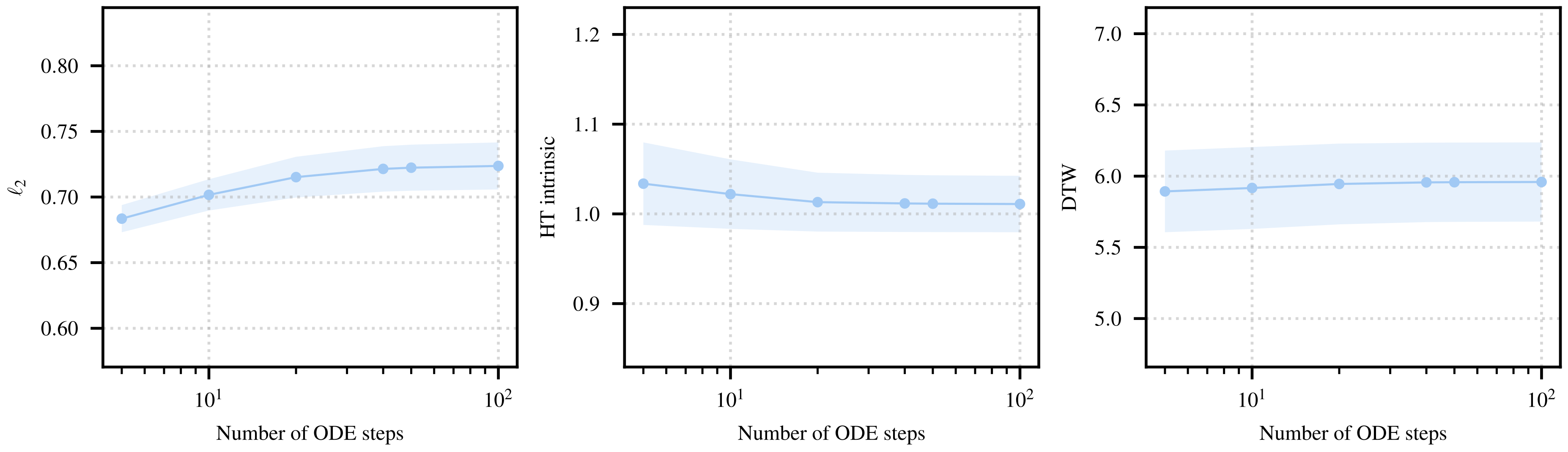}
  \caption{Solver-step ablation for cylindrical geodesic flow matching. Mean and $\pm 1$ standard deviation of $L_2$, HT intrinsic distance, and DTW are plotted against the number of ODE integration steps, aggregated across subjects and seeds.}
  \label{fig:steps-ablation}
\end{figure}

Figure~\ref{fig:steps-ablation} shows a short transient between $4$ and $\sim$16 steps after which all three metrics fall within the variability of the $100$-step reference; HT intrinsic distance and DTW are nearly flat, while $L_2$ improves slightly off the lowest step count before stabilizing. This is consistent with the analytic structure of the bridge: the oracle path is the closed-form constant-speed cylindrical geodesic rather than a learned curve with significant curvature, so faithful integration does not require fine temporal discretization. Inference budgets can therefore be reduced by roughly an order of magnitude without measurable degradation.

\subsection{Qualitative Examples}
\label{app:qualitative}
We provide side-by-side waveform comparisons on representative held-out examples from each dataset. Each figure shows the source waveform with the target in the leftmost panel, and overlays the prediction of each of the four learned models: the direct supervised predictor, the time-series affine bridge, the Hilbert naive bridge, and the cylindrical geodesic bridge. The predictions are overlaid on the same target in the remaining panels. Curves are color-coded throughout as initialization (\emph{Init}, gray), ground truth (\emph{Target}, green), and prediction (\emph{Pred}, orange).

\begin{figure*}[h]
    \centering
    \includegraphics[width=\textwidth]{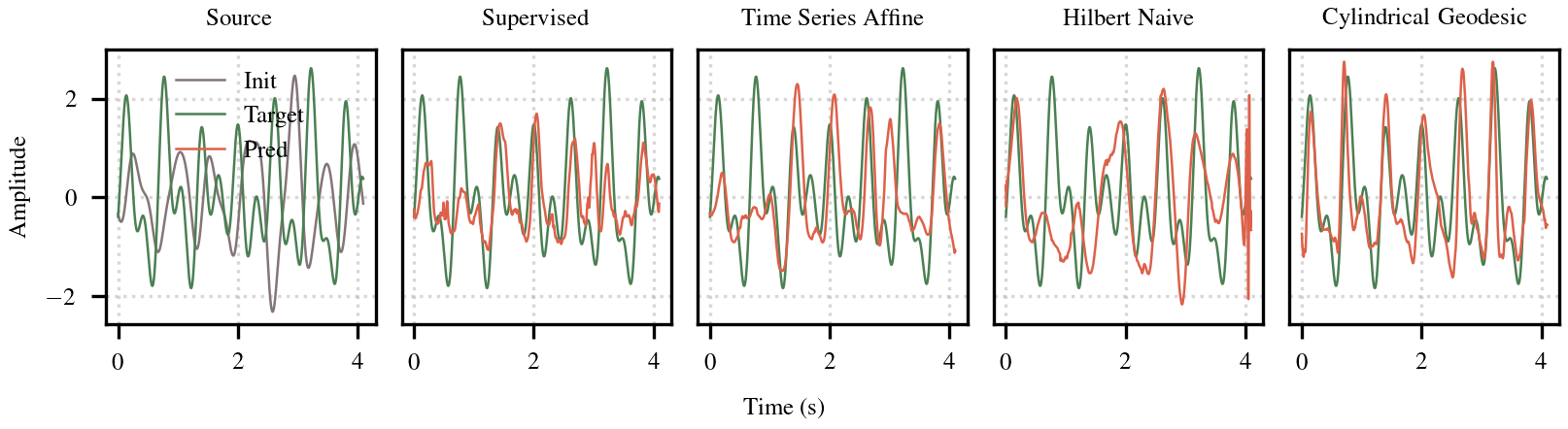}
    \caption{Qualitative comparison on a representative PPG zero-shot transfer example (4.096\,s window, channel-wise z-scored).}
    \label{fig:qualitative-ppg}
\end{figure*}

\begin{figure*}[h]
    \centering
    \includegraphics[width=\textwidth]{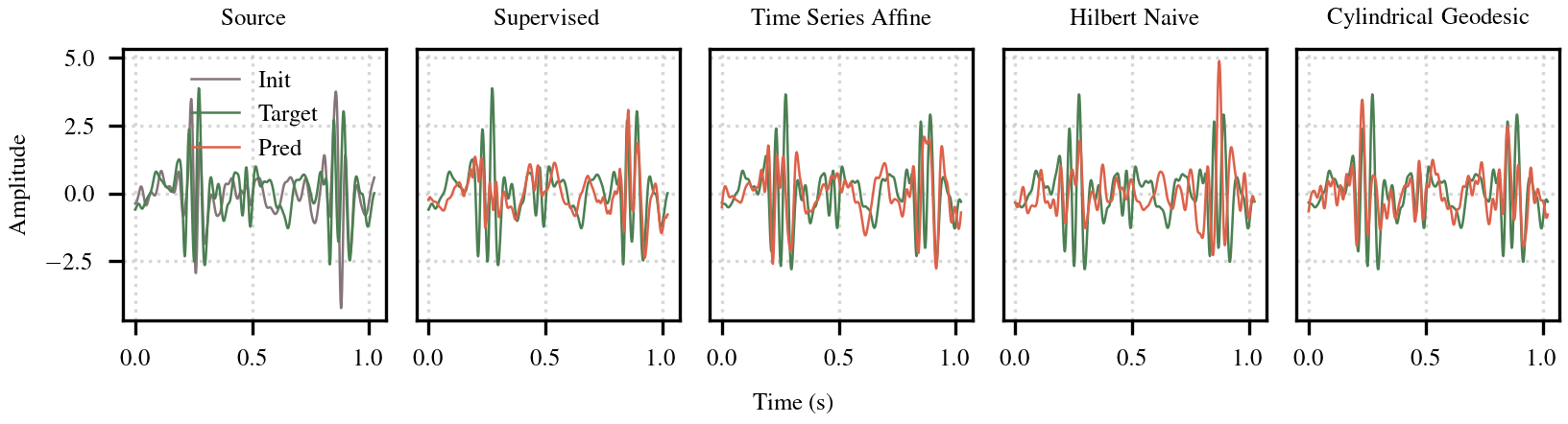}
    \caption{Qualitative comparison on a representative SCG limited-support adaptation example (1.024\,s beat window, channel-wise z-scored).}
    \label{fig:qualitative-scg}
\end{figure*}

The two examples illustrate complementary failure modes of the baseline bridges. On PPG (Figure~\ref{fig:qualitative-ppg}), where the source-to-target gap manifests primarily as morphology and timing differences in low-frequency oscillations, the cylindrical geodesic prediction follows successive systolic peak heights and zero crossings closely, whereas the Hilbert naive bridge exhibits a visible boundary artifact at the right edge of its panel and the time-series affine bridge attenuates the dominant peaks. On SCG (Figure~\ref{fig:qualitative-scg}), where the target contains sharp, locally concentrated peaks separated by quieter intervals, the cylindrical geodesic path preserves both the location and the amplitude of these peaks, while the affine and Hilbert naive bridges under- or over-shoot them and the supervised predictor produces a smoother envelope that softens the rapid transients.

\newpage
\section*{NeurIPS Paper Checklist}
 

\begin{enumerate}

\item {\bf Claims}
    \item[] Question: Do the main claims made in the abstract and introduction accurately reflect the paper's contributions and scope?
    \item[] Answer: \answerYes{} 
    \item[] Justification: The four contributions listed in the introduction are supported by Sections~\ref{subsec:cylindrical-representation}--\ref{subsec:intrinsic-parameterization}, Appendix~\ref{app:proofs}, and Tables~\ref{tab:ppg-main-results}--\ref{tab:scg-main-results}.

\item {\bf Limitations}
    \item[] Question: Does the paper discuss the limitations of the work performed by the authors?
    \item[] Answer: \answerYes{} 
    \item[] Justification: Three limitations are discussed in Section~\ref{sec:discussion}: metric choice, analytic-signal artifacts near low-amplitude segments, and single-institution evaluation.

\item {\bf Theory assumptions and proofs}
    \item[] Question: For each theoretical result, does the paper provide the full set of assumptions and a complete (and correct) proof?
    \item[] Answer: \answerYes{} 
    \item[] Justification: Theorem~\ref{thm:embedding}, Proposition~\ref{prop:geodesic}, and Theorems~\ref{thm:fidelity}--\ref{thm:length} state all assumptions explicitly and include complete proofs in Appendix~\ref{app:proofs}.

    \item {\bf Experimental result reproducibility}
    \item[] Question: Does the paper fully disclose all the information needed to reproduce the main experimental results of the paper to the extent that it affects the main claims and/or conclusions of the paper (regardless of whether the code and data are provided or not)?
    \item[] Answer: \answerYes{} 
    \item[] Justification: Appendices~\ref{app:dataset-details}--\ref{app:model-details} specify architecture, optimizer, hyperparameters, exact subject splits (Table~\ref{tab:subject-splits}), and compute environment.

\item {\bf Open access to data and code}
    \item[] Question: Does the paper provide open access to the data and code, with sufficient instructions to faithfully reproduce the main experimental results, as described in supplemental material?
    \item[] Answer: \answerNo{} 
    \item[] Justification: The datasets were collected under IRB-approved protocols and contain protected health information; code will be released upon acceptance.

\item {\bf Experimental setting/details}
    \item[] Question: Does the paper specify all the training and test details (e.g., data splits, hyperparameters, how they were chosen, type of optimizer) necessary to understand the results?
    \item[] Answer: \answerYes{} 
    \item[] Justification: Section~\ref{subsec:experimental-setup} summarizes tasks, metrics, and baselines; Appendices~\ref{app:dataset-details}--\ref{app:model-details} provide full preprocessing, split, and training details.

\item {\bf Experiment statistical significance}
    \item[] Question: Does the paper report error bars suitably and correctly defined or other appropriate information about the statistical significance of the experiments?
    \item[] Answer: \answerYes{} 
    \item[] Justification: All results report mean $\pm$ standard deviation over four subject-disjoint random splits (Tables~\ref{tab:ppg-main-results}--\ref{tab:scg-main-results}).

\item {\bf Experiments compute resources}
    \item[] Question: For each experiment, does the paper provide sufficient information on the computer resources (type of compute workers, memory, time of execution) needed to reproduce the experiments?
    \item[] Answer: \answerYes{} 
    \item[] Justification: Appendix~\ref{app:model-details} reports GPU type, VRAM, and per-run wall-clock times.
    
\item {\bf Code of ethics}
    \item[] Question: Does the research conducted in the paper conform, in every respect, with the NeurIPS Code of Ethics \url{https://neurips.cc/public/EthicsGuidelines}?
    \item[] Answer: \answerYes{} 
    \item[] Justification: All data were collected under IRB-approved protocols with informed consent.

\item {\bf Broader impacts}
    \item[] Question: Does the paper discuss both potential positive societal impacts and negative societal impacts of the work performed?
    \item[] Answer: \answerYes{} 
    \item[] Justification: A broader impact paragraph in Section~\ref{sec:discussion} discusses both positive (expanded wearable monitoring) and negative (silent translation failures, re-identification risk) societal impacts.
    
\item {\bf Safeguards}
    \item[] Question: Does the paper describe safeguards that have been put in place for responsible release of data or models that have a high risk for misuse (e.g., pre-trained language models, image generators, or scraped datasets)?
    \item[] Answer: \answerNA{} 
    \item[] Justification: The paper does not release pre-trained models or datasets that pose a high risk for misuse.

\item {\bf Licenses for existing assets}
    \item[] Question: Are the creators or original owners of assets (e.g., code, data, models), used in the paper, properly credited and are the license and terms of use explicitly mentioned and properly respected?
    \item[] Answer: \answerNA{} 
    \item[] Justification: The paper uses only standard open-source libraries (PyTorch, NeuroKit2) and in-house collected data; no external datasets or pre-trained models are used.

\item {\bf New assets}
    \item[] Question: Are new assets introduced in the paper well documented and is the documentation provided alongside the assets?
    \item[] Answer: \answerYes{} 
    \item[] Justification: The PPG and SCG datasets introduced in this work are documented in Appendix~\ref{app:dataset-details}, including hardware, cohort composition, preprocessing, and consent procedures.

\item {\bf Crowdsourcing and research with human subjects}
    \item[] Question: For crowdsourcing experiments and research with human subjects, does the paper include the full text of instructions given to participants and screenshots, if applicable, as well as details about compensation (if any)? 
    \item[] Answer: \answerYes{} 
    \item[] Justification: Data were collected under IRB-approved protocols with written informed consent and compensation from all participants (Appendix~\ref{app:dataset-details}).

\item {\bf Institutional review board (IRB) approvals or equivalent for research with human subjects}
    \item[] Question: Does the paper describe potential risks incurred by study participants, whether such risks were disclosed to the subjects, and whether Institutional Review Board (IRB) approvals (or an equivalent approval/review based on the requirements of your country or institution) were obtained?
    \item[] Answer: \answerYes{} 
    \item[] Justification: Both datasets were collected under separate IRB-approved protocols at the authors' institution (Appendix~\ref{app:dataset-details}).

\item {\bf Declaration of LLM usage}
    \item[] Question: Does the paper describe the usage of LLMs if it is an important, original, or non-standard component of the core methods in this research? Note that if the LLM is used only for writing, editing, or formatting purposes and does \emph{not} impact the core methodology, scientific rigor, or originality of the research, declaration is not required.
    \item[] Answer: \answerNA{} 
    \item[] Justification: LLMs were used only for writing and editing, not as a component of the core methodology.

\end{enumerate}

\end{document}